%% file: main.tex
\def\year{2022}\relax
\documentclass[letterpaper]{article} 
\usepackage{aaai22}  
\usepackage{times}  
\usepackage{helvet}  
\usepackage{courier}  
\usepackage[hyphens]{url}  
\usepackage{graphicx} 
\usepackage{natbib}  
\usepackage{caption} 
\DeclareCaptionStyle{ruled}{labelfont=normalfont,labelsep=colon,strut=off} 
\usepackage[ruled,vlined]{algorithm2e}
\usepackage{booktabs}       
\usepackage{array}
\usepackage{multirow}
\usepackage{float}
\usepackage{bbm}
\usepackage{comment}

\input{math_commands}

\newcommand\method{SHPI}

\title{Improving Long-Term Metrics in Recommendation Systems using Short-Horizon Reinforcement Learning}
\author{
    Bogdan Mazoure\textsuperscript{\rm 1,*},
    Paul Mineiro \textsuperscript{\rm 2},
    Pavithra Srinath \textsuperscript{\rm 2},
    Reza Sharifi Sedeh \textsuperscript{\rm 3,*},
    Doina Precup \textsuperscript{\rm 1,4},
    Adith Swaminathan \textsuperscript{\rm 2}
}
\affiliations{
    \textsuperscript{\rm 1} McGill University \& MILA\\
    \textsuperscript{\rm 2} Microsoft Research\\
    \textsuperscript{\rm 3} Facebook\\
    \textsuperscript{\rm 4} DeepMind\\
    bogdan.mazoure@mail.mcgill.ca
}

\begin{document}
\nonfrenchspacing

\maketitle

\begin{abstract}
We study session-based recommendation scenarios where we want to recommend items to users during sequential interactions to improve their long-term utility. 
Optimizing a long-term metric is challenging because the learning signal (whether the recommendations achieved their desired goals) is delayed and confounded by other user interactions with the system. 
Targeting immediately measurable proxies such as clicks can lead to suboptimal recommendations due to misalignment with the long-term metric.
We develop a new reinforcement learning algorithm called Short Horizon Policy Improvement (\method) that approximates policy-induced drift in user behavior across sessions. 
\method~is a straightforward modification of episodic RL algorithms for session-based recommendation, that additionally gives an appropriate termination bonus in each session. 
Empirical results on four recommendation tasks show that \method~can outperform state-of-the-art recommendation techniques like matrix factorization with offline proxy signals, bandits with myopic online proxies, and RL baselines with limited amounts of user interaction.
\end{abstract}

\input{intro}

\input{related}

\input{problem_setting}

\input{algo}

\input{expt}

\input{disc}


\input{main.bbl}
\onecolumn
\newpage
\section*{Appendix}
\input{appendix}
\end{document}

%% file: math_commands.tex
\usepackage{amsmath,amsfonts,bm}

\def\eqref#1{equation~\ref{#1}}

\def\1{\bm{1}}

\DeclareMathAlphabet{\mathsfit}{\encodingdefault}{\sfdefault}{m}{sl}
\SetMathAlphabet{\mathsfit}{bold}{\encodingdefault}{\sfdefault}{bx}{n}

\DeclareMathOperator*{\argmax}{arg\,max}
\DeclareMathOperator*{\argmin}{arg\,min}

\newcommand{\cD}{\mathcal{D}}

\newtheorem{theorem}{Theorem}[section]

\newtheorem{lemma}[theorem]{Lemma}

\newtheorem{assumption}[theorem]{Assumption}
\newtheorem{proof}[theorem]{Proof}

\renewcommand{\vec}[1]{\ensuremath{\bm{#1}}}

\newcommand{\mat}[1]{\ensuremath{\mathbf{#1}}}

%% file: intro.tex
\section{Introduction}
\label{sec:intro}

High quality recommendations can provide personalized experiences for users and drive increased usage, revenue and utility for service providers. 
These recommendations do not occur as 
isolated interactions, but rather via repeated sequential interactions with users. 
For instance, mobile health interventions nudge users over time to achieve their long-term goals~\citep{liao2020personalized}, e-commerce sites personalize product suggestions within user sessions to maximize their likelihood of purchases~\citep{ma2020temporal}, and marketing campaigns target user populations to optimize their sales. 

Traditionally, recommender systems use immediately measurable user feedback (e.g., implicit feedback like clicks or explicit feedback like item ratings) to identify good recommendations. 
For sequential interactions driving towards long-term goals these myopic metrics are imperfect proxies for Long-Term Rewards (LTR).
Consider an e-commerce site optimizing its recommendations for total monthly conversions by its user population: 
\begin{equation*}
    LTR = \frac{Usage}{Month} = \frac{Reward}{Rec}\frac{Recs}{Session}\frac{Sessions}{Month}.
\end{equation*}
Recommenders that focus on immediate user feedback are too myopic for maximizing LTR because they greedily optimize the first term even if that lowers the other two terms.

\begin{figure}[htb]
    \centering
    \includegraphics[width=0.75\linewidth]{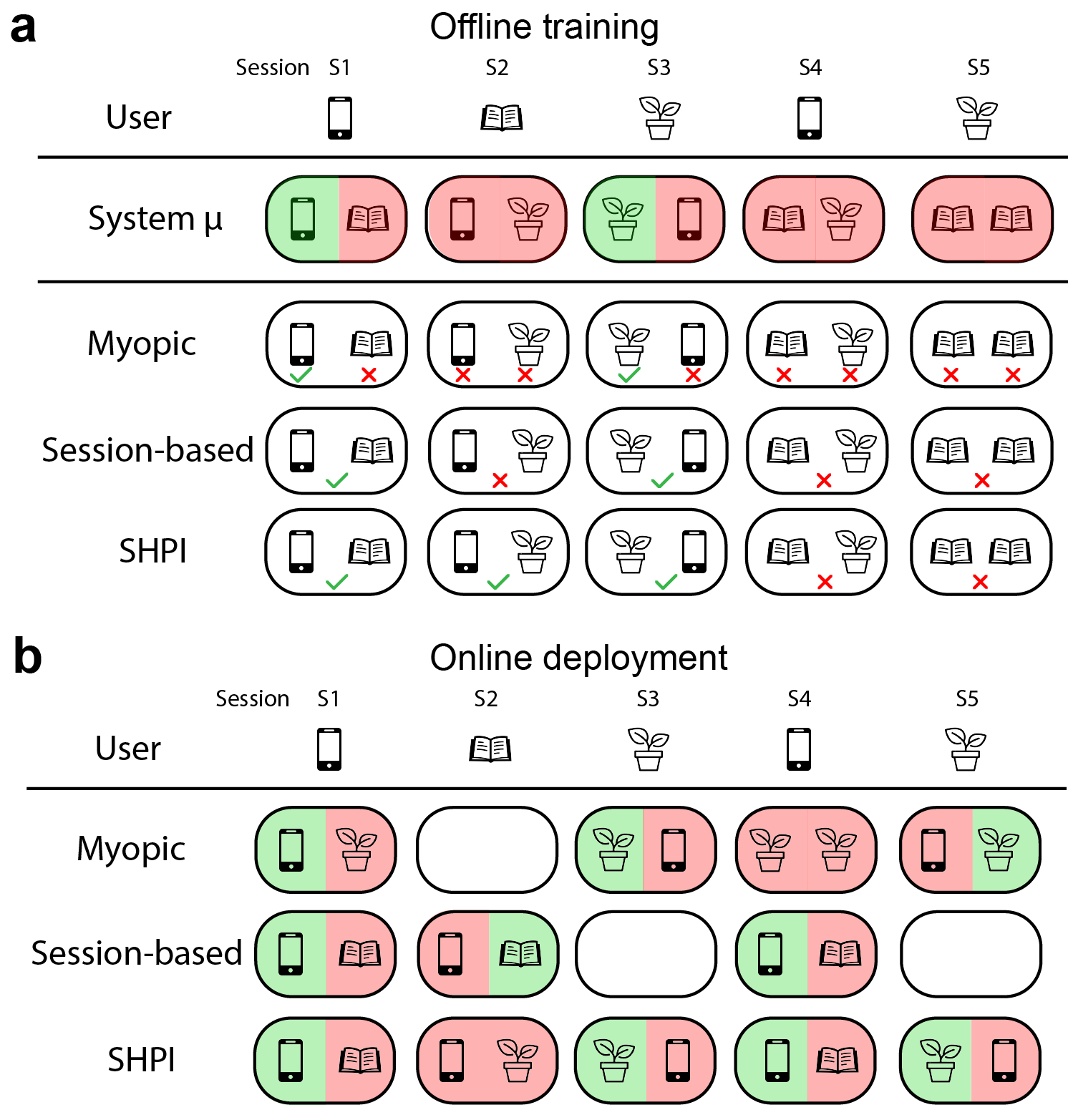}
    \caption{Stylized example of e-commerce recommendations. A session $S_i$ lasts for two interactions with conversions (green) or not (red). (a) Myopic bandit, Session-based RL and \method~infer different training signals in logged data collected from $\mu$. \method~identifies that good discovery recommendations in $S_2$ can increase the number of sessions and thereby total conversions. (b) Myopic bandit learns that two items have $50\%$ conversion rate and always recommends them, missing out on rare sessions $S_2$. Session-based RL succeeds in all its sessions including $S_2$ 
    but does not make discovery recommendations, thereby missing out on sessions $S_3, S_5$. \method~trades-off between success in a session ($S_2$) and initiating more sessions.}
    \label{fig:example_recommendation}
\end{figure}

Consequently, reinforcement learning (RL) techniques have been applied to directly optimize recommendations for LTR~\citep{zheng2018drn}. 
These techniques require many user interactions, 
with the sample complexity of RL growing with the horizon of the problem~\cite{dann2015sample,yin2021near}. 
Fortunately, user interactions in many recommendation applications typically occur in short bursts called sessions. 
Episodic RL algorithms are well-suited for optimizing session-based recommendations because each session is treated as an independent episode, and the typical length of a session is small. 
However, as we will show, episodic RL for session-based recommendations can still be too myopic and we propose a fix called \method.

Fig.~\ref{fig:example_recommendation} illustrates an example where each session length is $2$ and there are a maximum of $5$ sessions in a month. 
When users visit the site, they have a limited amount of patience to interact and find an appropriate product.
Within a user session,
we may recommend attractive
products which advertises our item inventory for a user's future shopping needs or we may recommend more functional products that address their immediate need. 

Consider first a deployed system $\mu$ that recommends items uniformly at random. All items have some non-zero user exposure and so users initiate sessions for all their intents. However, uniform random is not a good recommendation policy so $\mu$ sees very few conversions.

Next consider a recommender that has been tuned to optimize for conversions myopically using data collected from $\mu$. It infers that there is a small subset of items that have good conversion rates and only explores these most popular items. However, users exposed to only a small subset of items may never initiate sessions for other intents, leading to fewer sessions and eventually sub-optimal total conversions. 

Suppose instead we optimized session-based recommendations using episodic RL, where each episode corresponds to a user session. 
Such a policy would recommend functional items which increase the chance of a conversion within each session. This is an optimal policy \emph{if} the distribution of sessions matches what is seen in the training data. Unfortunately, the deployed policy can influence which sessions are initiated by users. Therefore, although the deployed policy has a $100\%$ session conversion rate, it comes at the cost of fewer sessions overall. 

One non-solution is to ignore 
sessions and instead treat the duration of the LTR as the RL problem horizon. In the example above this would treat each month of user activity as an episode. However, the sample complexity of RL scales with episode length. How can we optimize for better session-based recommendations without paying the statistical costs for the long horizons of the LTR? 
Notice that the key difficulty is that the number of sessions in a month is influenced by the recommendation policy. 
In this paper, we approximate this dependence by first estimating long-term user behavior (e.g., $\frac{Sessions}{Month}$) 
separately using off-policy data collected from an existing recommender system. 
Then, by invoking episodic RL algorithms with an additional bonus on episode termination to capture policy-induced user shifts, we develop a practical algorithm called Short Horizon Policy Iteration (\method). \method~attempts to improve LTR by exploiting the session structure of user behavior but without making the strong assumption that sessions can be treated as independent episodes. 

We conduct experiments in four relevant domains: a task where the reward surface is a challenging non-convex function, a recommender system benchmark with a sub-modular reward function~\citep{rohde2018recogym}, a real-world private dataset and an HIV treatment domain with delayed rewards~\citep{ernst2006clinical}. 
Our experiments show that our algorithm is simple to run, can be adapted to work with off-policy data, and outperforms session-based RL, contextual bandit and other recommendation baselines. 


%% file: related.tex
\section{Related Work}
\label{sec:related}

\paragraph{Recommendations for LTR.}
Prior work on long-term signals for personalized recommendations has mostly studied online methods: \citet{pan2019policy} and \citet{ma2020off} use the actor-critic framework to tackle the problem 
while \citet{wu2017returning} combined short-term and long-term objectives based on clicks for improving customer conversion. 
Our work builds on \citet{wu2017returning} who identified a trade-off between exploring new items, exploiting items for immediate utility, and exploiting other items to bring back returning users. 
Our work extends their reasoning from bandits to session-based episodic RL settings, applying both in online and offline scenarios, and generalizes their approach using episodic termination bonuses. 

Session-based recommendations are a widely-studied challenging task: most state-of-the-art approaches rely on session summarization mechanisms (e.g. attention or RNN) to capture long sequences of user interactions~\citep{li2017neural,jannach2017recurrent,liu2018stamp,zhang2020preference}. Despite advances in deep neural architectures for session-based recommendations, evidence suggests simpler approaches are competitive ~\citep{ludewig2018evaluation}.

Apart from supervised sequence models, reinforcement learning agents are another commonly used class of session-based recommender systems. A number of RL approaches have been proposed through the past years, but they rely on access to an environment simulator~\citep{shani2005mdp}.

\paragraph{Offline/Batch RL techniques.} 
Successful applications of offline RL agents in recommendation tasks constrain the state and action spaces, ensure proper diversity of recommendations, and balance interpolation and extrapolation errors~\citep{garg2020batch}.  Batch RL methods combat extrapolation error 
by either regularizing towards the data or optimizing objectives that are robust to distribution shift. For example, SPIBB~\citep{laroche2019safe} and MBS~\citep{liu2020provably} estimate a coverage measure over the state-action space, and define conservative updates when this measure falls below a 
threshold. 
Our method \method~is easily composable with many of these techniques. 

\paragraph{Truncated horizon reasoning in other settings.} GAE~\citep{schulman2015high} is a widely used online RL algorithm. 
\cite{su2020adaptive} used a truncated horizon estimator for off-policy evaluation and~\citet{kalweit2019composite} proposed a truncation for off-policy online control. Recent work such as THOR~\citep{sun2018truncated} and MAMBA~\citep{cheng2020policy} also use a truncated horizon for imitation learning. 
 Previous works have studied the impact of choosing higher discount factors (another way of truncating horizon) on planning~\citep{jiang2015dependence} and control~\citep{romoff2019separating}. 
Recent work~\citep{HuRL} establishes sufficient conditions for truncated horizon reasoning that are also needed for \method. All of these works are in the online interactive setting distinct from our batch RL setting.

%% file: problem_setting.tex
\section{Background}
\label{ref:problem_setting}

 \paragraph{Notation.} Lowercase letters $x$ represent either random variables or their realizations (depending on context), uppercase bold letters $\mat{A}$ represent matrices, lowercase bold letters $\vec{v}$ represent vectors, calligraphic letters $\mathcal{T}$ represent tensors and/or spaces. Index notation is the equivalence $w_{a:b}=\{w_a,w_{a+1},..,w_{b-1},w_{b}\}$. For a space $\mathcal{X}$, $\Delta(\mathcal{X})$ represents the set of distributions with support $\mathcal{X}$. 
 
 \subsection{Recommendation as Episodic MDP}
The notation used throughout this paper follows a single user $u$, where dependence is specified by subscript $X_u$; generalization to a user population is accomplished using user-specific features $X_u$ in the user context. The user's preferences are assumed to be stationary across time. 
There is also a notion of a task horizon, e.g., one month. We seek a recommendation policy that maps user contexts to valid recommendations so as to maximize total utility over the horizon. 
We formulate this problem as an episodic Markov Decision Process (MDP) with a long horizon.
The results in this paper directly translate to infinite horizon discounted objectives. 

An episodic MDP is defined by the tuple $(\mathcal{X}, \mathcal{A}, r, \mathcal{T}, p_0, \gamma, T)$ where $\mathcal{X}$ is the state space, $\mathcal{A}$ the action space, $r:\mathcal{X}\times \mathcal{A}\to \mathbb{R}$ the reward function, $\gamma\in [0,1]$ the discount factor and $T$ the maximum length of an episode. $\mathcal{T}:\mathcal{X}\times \mathcal{A}\to \Delta(\mathcal{X})$ is a Markov transition function and $p_0\in \Delta(\mathcal{X})$ is the initial state distribution. 

Unlike a general MDP, recommendation scenarios have additional structure, which can be represented using a factorization of the context or state space $x = [x_{s},x_{a},x_{u}]$ where  
$x_{s}$ encodes within-session context, 
$x_{a}$ encodes  action-dependent features, for example the valid recommendations in that context, 
and $x_{u}$ encodes across-session user profiles. 
Suppose a user interacted with the system on days $1,2$ and days $15,16,17$ in a month. In this formulation $T=31$, $x_{a,t}$ would mask out all actions for $t \neq 1,2,15,16,17$. 
$x_{s,1}, x_{s,15}$ would encode session-initiation and $x_{s,2}, x_{s,17}$ encode session-termination. 
If we observe only session-based rewards (e.g., conversions on an e-commerce site) they can be encoded using $r(x, a) = r(x_s)$. 


At each $t=0,1,2,..,T$, the system agent observes $x_t\in \mathcal{X}$ and picks an action $a_t\in \mathcal{A}$ based on some behavior policy $\mu:\mathcal{X}\to \Delta(\mathcal{A})$. The goal of a learning agent is to maximize the expected cumulative rewards after $T$ timesteps. 
Finally we assume access to a historical dataset $\mathcal{D}$
of (context, action, feedback) trajectories collected from an existing system $\mu$  
(for instance, one that is tuned using bandit algorithms for optimizing clicks). We must use $\mathcal{D}$ to find a new policy $\pi$ that is reliably better than $\mu$. 
\subsection{RL Solution Concepts}

Let $\mathbb{P}^\mu_{0:t}$ denote ``rolling-in'' with $\mu$: $x_0 \sim p_0, a_0 \sim \mu(x_0), \dots x_t \sim \mathcal{T}(\cdot|x_{t-1}, a_{t-1})$ and $\mathbb{P}^\mu_{t}$ denote ``rolling out'' with $\mu$: 
 $\mathbb{P}^\mu_t=\mathbb{E}_{\mu}[\mathcal{T}(\cdot|x_{t-1},a_{t-1})]$.
 
The \emph{state value function} under any policy $\pi$ is  $$V^\pi_t(x_t)=\mathbb{E}_{\mathbb{P}^\pi_t}[\sum_{m=0}^{T-t}\gamma^m r(x_{t+m},a_{t+m})|x_t],$$ the \emph{state-action value function} is $$Q^\pi_t(x_t,a_t)=\mathbb{E}_{\mathbb{P}^\pi_t}[\sum_{m=0}^{T-t}\gamma^m r(x_{t+m},a_{t+m})|x_t,a_t],$$ and the \emph{advantage function} is $A^\pi_t(x_t,a_t)=Q^\pi_t(x_t,a_t)-V^\pi_t(x_t)$. Note that we have time-indexed value and advantage functions;  however, we drop the index 
because when $T$ is very large these functions are time-invariant under the assumption that user preferences are stationary. 
The \emph{expected return} is $\eta^\pi=\mathbb{E}_{p_0}[V^\pi_0(x_0)]=\mathbb{E}_{\mathbb{P}^\pi_{0:t}}[\sum_{i=0}^{t-1} \gamma^i r_i + V^\pi_t(x_t)]=\mathbb{E}_{\mathbb{P}^\pi_{0:t}}[\sum_{i=0}^{t-1} \gamma^i r_i + Q^\pi_t(x_t, a_t)]$ for $1 \le t \le T$.

The online/interactive/on-policy RL problem is to find a $\pi$ by interacting with users such that $\eta^{\pi}$ is near-optimal. The offline/batch RL problem attempts to find such a $\pi$ using only access to the logged dataset $\mathcal{D}$. In most practical recommendation settings, we have a hybrid scenario where we have a historical dataset $\mathcal{D}$ and also have a small budget for online user interactions. In Sec.~\ref{sec:algo} we provide practical online and offline RL algorithms for both settings. 

 \paragraph{Policy Gradient}
 A typical interactive RL strategy attempts to maximize $\eta^{\pi} \equiv \argmax \mathbb{E}_{\mathbb{P}^\pi
 }[A^\pi
 (x
 ,a
 )]$ 
 by collecting samples from $\pi$ interacting with the environment to estimate the expectation $\mathbb{P}^\pi
 $ and the advantages $A^\pi
 $.
 If online learning is impossible due to the absence of a simulator or is too costly (e.g., in medical treatment domains), an improved policy $\pi$ can still be found for a given dataset gathered by behavior policy $\mu$. One approach uses ideas inspired by the policy gradient of $\nabla \eta^\mu$, and essentially maximizes $\mathbb{E}_{\mathbb{P}^\mu
 }[A^\mu
 (x
 ,a
 )]$. However the policy improvement step when following this strategy tends to be very small (see~\citet{kakade2002approximately} for a discussion). 
  Moreover, when we need to make a sequence of good recommendations within a session for user engagement or conversion, then one-step deviations from $\mu$ as estimated via $\hat{Q}^\mu(x_{t}, a^*) \approx 0$ (even for a good recommendation $a^*$). 
  So policy gradient techniques may completely miss avenues for improvement over $\mu$. 
  
 \paragraph{Conservative Policy Improvement}
 Another strategy uses a performance difference lemma~\citep{kakade2002approximately} to relate $\eta^{\pi}-\eta^{\mu} = \mathbb{E}_{\mathbb{P}^\mu
 }[-A^\pi
 (x
 , a
 )] = \mathbb{E}_{\mathbb{P}^\pi
 }[A^\mu
 (x
 ,a
 )]$. Batch RL techniques that attempt to maximize policy improvement by optimizing $\mathbb{E}_{\mathbb{P}^\mu
 }[A^\pi
 (x
 ,a
 )]$ face the challenge of extrapolation errors, since $A^\pi
 $ estimated from samples collected from $\mu$ may not be reliable in important regions of the state-action space queried by $\pi$. Techniques that optimize $\mathbb{E}_{\mathbb{P}^\pi
 }[A^\mu
 (x
 ,a
 )]$ suffer from distribution shift since samples collected from $\mu$ may not reliably estimate expectations w.r.t $\mathbb{P}^\pi
 $.
 

\paragraph{Model-based RL}
 We can 
 separately estimate the instantaneous reward $r$ and user transition dynamics $\mathcal{T}$ using empirical statistics observed in $\mathcal{D}$. Given estimates of $\hat{r}$ and $\hat{\mathcal{T}}$, we can use approximate dynamic programming or other planning methods to decode a near-optimal $\hat{\pi}$ in the approximate MDP. The statistical and computational difficulty of these approaches scales with horizon $T$.
  Model-based RL approaches are impractical for our recommendation setting because we would need to estimate $\mathcal{T}$, which describes highly stochastic dynamics of high-dimensional vectors $\mathcal{X}$. $r$ is potentially easier to estimate but we need reliable enough $(\hat{r}, \hat{\mathcal{T}})$ estimates to perform planning over long horizons $T \ge 100$. Such precise estimates require an impractically large dataset $\mathcal{D}$.
  
 \paragraph{Truncated horizon}
 We can solve an artificially shorter horizon RL problem, by either treating $H < T$ as effective task horizon or by decreasing the discount factor to $\gamma' < \gamma$. User sessions are a natural and popular choice for defining truncated episodes of length up to $H$. 
 This essentially assumes that the state $x_{H+1} \sim p_0$ whereas in reality $x_{H+1} \sim \mathcal{T}(\cdot| x_H, a_H)$. We saw in Sec.~\ref{sec:intro} that when policy actions within a user session 
 influence the user population who initiate future sessions, $\mathbb{P}^\pi[x_{H+1}] \neq p_0$.
 


 
 
We now identify an alternative solution strategy for truncated horizon reasoning 
 that leverages the session structure of user interactions without assuming that $\mathbb{P}^\pi[x_{H+1}] \approx p_0$. 

%% file: algo.tex
\section{Algorithm}
\label{sec:algo}

Can we estimate how the distribution over sessions depends on any particular recommendation policy? 
In some very simple cases, the answer is yes. For instance, suppose a user $u$ had a known, intrinsic \emph{session initiation rate} $c(x_u)$. Then estimating $\mathbb{P}^\pi[x_{u}]$ and computing $\mathbb{E}_{\mathbb{P}^\pi(x_u)}\left[ c(x_u) \right]$ gives us the correct policy-dependent session distribution. 
Our key insight is that we can learn a suitable bonus function like $c(x_u)$ that, when invoked on the distribution of terminal states produced by an episodic RL agent, gives the correct long-term signal for learning. 
For example, suppose we are given the state value function for an optimal policy $\pi^*$ for our problem, $V^*(x)$. Bellman optimality equations state that the $\pi$ that maximizes the instantaneous re-shaped reward, $\pi(x) = \argmax_{a \in \mathcal{A}} \{ r(x,a) + \gamma \mathbb{E}_{x' \sim \mathcal{T}(\cdot|x,a)}\left[ V^*(x')\right] \}$ is an optimal policy. In other words, if we use $V^*$ as a bonus function then even a myopic bandit algorithm that optimizes a re-shaped reward will recover the optimal policy.  

Since we do not know $V^*$, we begin by first estimating $V^\mu$ with classic policy evaluation techniques to use as a bonus instead  (Sec.~\ref{sec:eval}). Then we incorporate these estimates in an online on-policy iteration scheme (Sec.~\ref{sec:online_shpi}) and later extend it to the offline off-policy setting (Section~\ref{sec:offline_shpi}).

\subsection{Estimating Termination Bonus $V^\mu$}
\label{sec:eval}
Given a dataset $\mathcal{D} = \{ (x_i \sim \mathbb{P}^\mu[x], a_i \sim \mu(\cdot|x_i), {x'}_i \sim \mathcal{T}(\cdot|x_i,a_i), r_i \sim r(x,a)) \}_{i=1}^n$, we estimate $V^\mu(x)$ using Bellman residual minimization: 
\begin{equation}
    \hat{V}^\mu=\argmin_{V}\sum_{(x,a,r,x')\in\mathcal{D}}(r+\gamma V(x')-V(x))^2\;.
    \label{eq:v_dm_obj}
\end{equation}
In practice we use SGD on $\mathcal{D}$ to optimize Eq.~\ref{eq:v_dm_obj}.




\subsection{Online Short-Horizon Policy Iteration}
\label{sec:online_shpi}


We next combine Monte Carlo samples collected from online interactions along with the evaluated $V^\mu$ model 
to estimate 
$k$-step advantages: 
\begin{equation}
\begin{split}
    \mathbb{A}^\pi_{(k),t}(x_t, a_t)=\mathbb{E}_{\mathbb{P}^\pi_{t:t+k} }[&\sum_{m=0}^{k-1}\gamma^mr(x_{t+m},a_{t+m})\\
     +\gamma^k&V^\mu_{t+k}(x_{t+k})|x_t,a_t]-V^\mu_t(x_t).
\end{split}
\label{eq:mc_online}
\end{equation}
These advantages estimate the effect of starting in state $x_t$ at time $t$, taking action $a_t$, then rolling-in for $(k-1)$-steps with $\pi$ before reverting to $\mu$ till the end of the episode (remember that episode length is $T$). Furthermore the baseline that is subtracted from this $Q$-function is the estimated value of following $\mu$ from state $x_t$.  
Note that $\mathbb{A}^\pi_{(1),t} \equiv A^\mu_t$ and if $k \ge T-t: \mathbb{A}^\pi_{(k),t} \equiv A^\pi_t$. 
So using $k=1$ recovers the advantages used by policy gradient schemes, while $k=T$ recovers advantages in conservative policy iteration schemes. 
In practice we set $k$ to be large enough to span 
user sessions, 
e.g., we set $k=3$ in the example of Sec.~\ref{sec:intro}. 

\begin{algorithm}[t]
\SetAlgoLined
\SetKwInOut{Input}{Input}
 \SetKwInOut{Output}{Output}
 \Input{Batch $\mathcal{D}$,  horizon $k$, model class $\mathcal{F}$, iterations $J$}
 
$\hat{V}^\mu\leftarrow$ Eq.(\ref{eq:v_dm_obj})\;
$\pi^{(0)}\leftarrow \mu$\;
\For{$j=1,..,J$}{
Collect $\tau=\{(x_t,a_t,r_t)\}_{t=1}^T$ using $\pi^{(j)}$\;
$\hat{\mathbb{A}}_{(k)}^\pi \leftarrow $ Eq.(\ref{eq:mc_online})\;

$
        \hat{f} \leftarrow \underset{f\in \mathcal{F}}{\argmin}\sum_{(x,a)\in \tau}\{f(x,a)-\hat{\mathbb{A}}_{(k)}^\pi(x,a)\}^2;\;
        $
        $\pi^{(j)}(x) \leftarrow \argmax_{a \in \mathcal{A}}{\hat{f}(x,a)}$\;
        \tcc{Optional: Update $V^\mu$ }
       $\cD \leftarrow \cD \cup \tau$\;
       $\hat{V}^\mu\leftarrow$ Eq.(\ref{eq:v_dm_obj})\;
}
 \caption{Online \method}
 \label{alg:online_shpi}
\end{algorithm}

We can use any episodic RL algorithm with these $k$-step advantages or termination bonuses. 
In this paper we instantiate a classic approximate policy iteration algorithm~\citep{bertsekas1995neuro} in Alg.~\ref{alg:online_shpi}. We start with an existing policy $\mu$ from either bandit-based or session-based RL. We first estimate $V^\mu$  
then we interleave advantage estimation using Eq.~\ref{eq:mc_online} with policy improvement steps. Note how the episodic RL algorithm is rewarded implicitly with the  termination bonus of $V^\mu(x_{k+1})$ inside the advantage estimates.  
Finally, we optionally update the termination bonus using the new data. Note that we recover session-based episodic RL algorithms if $V^\mu$ always returns $0$. And if additionally $k=1$ we recover contextual bandit algorithms. 

\subsection{Offline Short-Horizon Policy Iteration}
\label{sec:offline_shpi}


We now turn to the offline problem setting where we must recover a policy $\pi$ that improves over $\mu$ using only the dataset $\mathcal{D}$. Since the training dataset was collected by $\mu$ and not $\pi$, we use PDIS~\citep{precup2000eligibility} to rewrite the expectation in terms of $\mu$:
\begin{equation}
\begin{split}
    \mathbb{A}^\pi_{(k),t}(x_t, a_t)=\mathbb{E}_{\mathbb{P}^\mu_{t:t+k} }[&\sum_{m=0}^{k-1}w_{t+1}^{t+m}\gamma^mr(x_{t+m},a_{t+m})\\
    +w_{t+1}^{t+k}\gamma^k&V^\mu_{t+k}(x_{t+k})|x_t,a_t]-V^\mu_t(x_t),
    \label{eq:advantage_est}
\end{split}
\end{equation}
where $w_{t_1}^{t_2}=\prod_{m=t_1}^{t_2} \frac{\pi(a_{m}|x_{m})}{\mu(a_{m}|x_{m})}$ when $t_1 \le t_2$, $1$ otherwise. To mitigate potentially large values of $w_{t_1}^{t_2}$ we use weight clipping~\citep{ionides2008truncated}. The finite sample version of our estimator from $n$ Monte Carlo trajectories is
\begin{equation}
\begin{split}
    \hat{\mathbb{A}}^\pi_{(k),t}:=\frac{1}{n}\sum_{i=1}^n[&\sum_{m=0}^{k-1}\hat{w}_{i,t+1}^{t+m}\gamma^m\hat{r}(x_{i,t+m},a_{i,t+m})\\
    +\hat{w}_{i,t+1}^{t+k}\gamma^k&\hat{V}^\mu(x_{i,t+k})]-\hat{V}^\mu(x_{i,t})\;.
    \label{eq:advantage_MC}
\end{split}
\end{equation}
The off-policy API algorithm called Offline \method~follows immediately and is sketched in Alg.~\ref{alg:offline_shpi}. 

\begin{algorithm}[t]
\SetAlgoLined
\SetKwInOut{Input}{Input}
 \SetKwInOut{Output}{Output}
 \Input{Batch $\mathcal{D}$, horizon $k$, model class $\mathcal{F}$, iterations $J$, clipping coefficients $q_1,q_2$ }
 
 $\hat{V}^\mu\leftarrow$ Eq.(\ref{eq:v_dm_obj})\;
 
 \tcc{Randomly initialize $\pi^{(0)} \neq \mu$}
$\pi^{(0)}\sim \Delta(\mathcal{A})$ \;

 \For{$j=1,..,J$}{

\For{$\tau=\{(x_t,a_t,r_t)\}_{t=1}^T \sim \mathcal{D}$}{
$\hat{\mat{w}}\leftarrow \text{\texttt{clip}}(\{\prod_{m=t+1}^{t+k}\frac{\pi^{(j-1)}(a_{m}|x_{m})}{\mu(a_{m}|x_{m})}\}_{t=1}^T,q_1,q_2)$
}
$\hat{\mathbb{A}}_{(k)}^\pi \leftarrow $ Eq.(\ref{eq:advantage_MC}) \; 
$
       \hat{f} \leftarrow \underset{f\in \mathcal{F}}{\argmin}\sum_{(x,a)\in \mathcal{D}}\{f(x,a)-\hat{\mathbb{A}}_{(k)}^\pi(x,a)\}^2\; \label{eq:reg}
        $\;
        $\pi^{(j)}(x) \leftarrow \argmax_{a \in \mathcal{A}}{\hat{f}(x,a)}$\;

}
 \caption{Offline \method}
 \label{alg:offline_shpi}
\end{algorithm}

We answer two questions about \method: 
(i) Using perfect estimates of $\hat{w},\hat{r},\hat{V}^\mu$, does \method~guarantee a policy improvement (are the algorithms \emph{consistent}?) and (ii) if some state-action pairs are poorly covered by $\mu$, is it possible to find an improved policy (are the algorithms \emph{efficient})? In Sec.~\ref{sec:analysis} we show that \method~is both consistent and efficient.

\subsection{Analysis}
\label{sec:analysis}
Since \method~is essentially an approximate policy iteration scheme with a modified advantage function, it inherits the convergence issues of API with function approximation. In particular,
under the assumptions that the true $\mathbb{A}^\pi_{(k)}$ lies in $\mathcal{F}$ (realizability) and that the minimizer $f^*$ incurs low error over $\mathcal{D}$ (completeness), then 
Alg.~\ref{alg:offline_shpi} will converge to a fixed point $\tilde{\pi}$. In practice, we use a maximum iteration number $J$ and 
return $\tilde{\pi} = \pi^{(J)}$. 
Note also that we initialize Alg.~\ref{alg:offline_shpi} away from $\mu$ so that $A^\mu$ and $\mathbb{A}^{\pi^{0}}_{(k)}$ do not coincide. 

We now describe sufficient conditions for \method~to work well, in increasing order of their generality. 
These conditions build on a long line of work tracing back to Blackwell optimality in the known MDP setting~\citep{blackwell1962discrete,hordijk2002blackwell}. The first condition suggests that \method~returns a near-optimal policy in a special class of MDPs, regardless of the $V^\mu$ we use. The second condition builds on recent work~\citep{HuRL} to show that if $\hat{V}^\mu$ has favorable properties then \method~returns a near-optimal policy. 
Finally, the third condition suggests that by tuning $k$ we might achieve a better bias-variance trade-off and \method~can return a larger policy improvement over $\mu$ (but not necessarily optimal) than existing methods.  

\paragraph{Condition 1 (most restrictive)} If the MDP($\mathcal{T}, R, p_0, T, \gamma$) admits a Blackwell-optimal policy for some unknown $k^* < T$ (equivalently for a $\gamma^* < \gamma$), then there exists a $k < T$ such that \method~is near-optimal. 

This condition holds when user session initiations are independent of the recommendation policy or if the task/reward horizon $T$ is long but the environment is reset to $p_0$ after $H < T$. 
Since the definition of Blackwell optimality says that a policy that is optimal for $\gamma^*$ remains optimal for $\gamma \ge \gamma^*$, any episodic RL algorithm can find a near-optimal solution if $k \ge \frac{1-\gamma}{1-\gamma^*} T$. Both session-based RL as well as \method~will thus be consistent for large enough $k$. 


\paragraph{Condition 2 (less restrictive)} From Corollary 4.1 of \citep{HuRL} if $\hat{V}^\mu$ has a small error $\| \hat{V}^\mu - V^*\|_\infty \le \epsilon$ then the policy returned by \method~is near optimal, $ V^*-V(\pi_{\method}) \le \epsilon \frac{(1-\lambda \gamma)^2}{(1-\gamma)^2}$ where $\lambda=\frac{(k-1)T}{(T-1)k}$. \method~can thus be better than session-based RL 
when $\hat{V}^\mu$ has a smaller $l_\infty$ error than a constant function. This is true, for instance when $\mu$ is on the value improvement path towards $\pi^*$ or if $\hat{V}^\mu$ is an \emph{improvable} bonus function  (see~\citep{HuRL}). 

\paragraph{Condition 3 (least restrictive)}. For all $\pi$, there exists $1 \le k \le T$ such that the mean squared error of estimating $\mathbb{A}^\pi_{(k)}$ is lower than estimating $\mathbb{A}^\mu$ and 
$\mathbb{A}^\pi$. Therefore by tuning $k$, we can have a good bias-variance trade-off so that \method~finds the tightest lower bound on policy improvement. 
Experiments in Sec.~\ref{sec:expt_k} confirm that $k$ set to session lengths can 
give better estimates and thereby better policy improvements than policy gradient and conservative policy iteration. 

%% file: expt.tex
\section{Experiments}
\label{sec:expt}


How does {\method} perform against contextual bandits and RL algorithms? How robust is it to bias in termination bonuses $V^\mu$ and to the choice of $k$? 
To answer these questions, we conducted experiments on three benchmarks and a private dataset: \textbf{(i)} a synthetic recommendation problem with a complex reward function~\citep{styblinski1990experiments}, \textbf{(ii)} the RecoGym environment with a sub-modular reward function~\citep{rohde2018recogym}, and \textbf{(iii)} the HIV treatment simulator proposed in~\citet{ernst2006clinical}.  
We consider only the (more realistic) offline problem setting and report results for Offline SHPI (Algorithm~\ref{alg:offline_shpi}) in all experiments\footnote{Code to reproduce all results except the private dataset is available in the supplementary material.}. 
\subsection{Experimental setup}
On all four domains, we first pre-train an online agent (deep SARSA), using a large number of interaction trajectories. This agent is then corrupted with $\varepsilon$-greedy noise, to simulate various thresholds of realistic performance. For instance, we study $\varepsilon=0$ (optimal logging policy), $\varepsilon=0.3$ and $\varepsilon=1$ (random uniform logging policy). This new agent is then deployed in the environment to gather a fixed dataset $\mathcal{D}$ which are used to train all compared methods.
We use the weight clipping parameters $q_1 = 0.5$ and $q_2 = 2$ throughout.

\paragraph{Test domains} The first domain implements a toy recommendation scenario in which each of $|\mathcal{A}|=10$ actions, with randomly sampled action features, affect the current context $X_t$ of a user in an additive way. The policy's input context is a moving average of previous contexts; the reward function $r$ depends only on the accumulated $x_t$ and is characterized by the non-convex Styblinski-Tang function (see Appendix for details). 
 Fig.~\ref{fig:toy_schema} shows a sketch of this task. 
 
 \begin{figure}[t]
     \centering
     \includegraphics[width=0.7\linewidth]{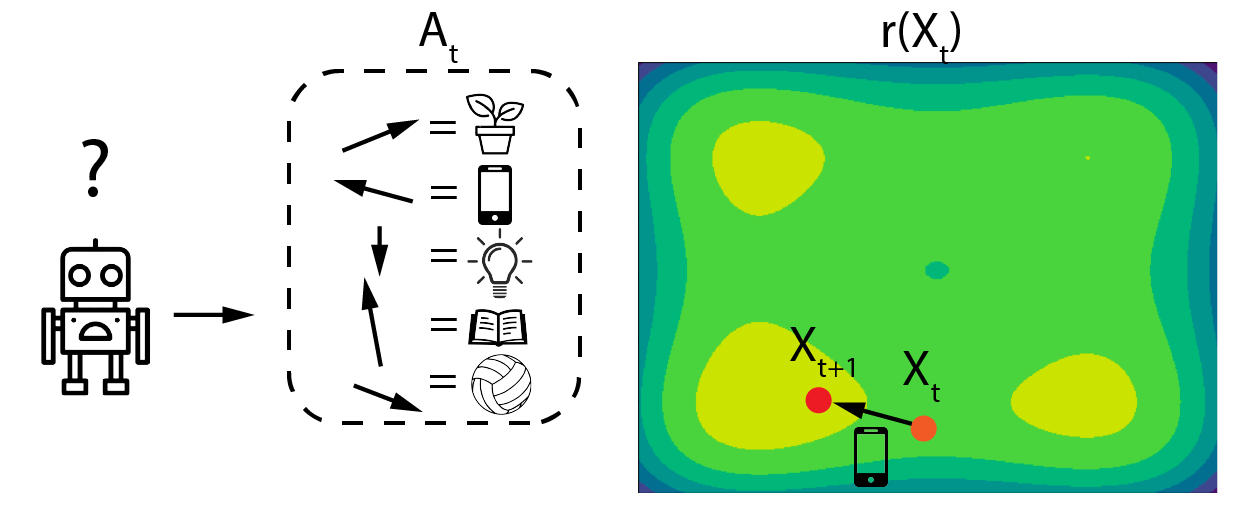}
     \caption{Schematic view of the toy recommendation task.}
     \label{fig:toy_schema}
 \end{figure}

The second benchmark, RecoGym, is a personalization domain used to study click-based recommender systems using RL. We added a long-term reward signal to the domain which is a submodular function that depends on all previously clicked products. 

The third benchmark is the HIV treatment simulator which relies on differential equations to simulate the administration of treatment to a group of patients. While it has relatively small state and action spaces, the HIV simulator is considered a proxy for real-world off-policy recommendation in the healthcare domain~\cite{liu2018representation,liu2020provably}.

Our fourth and last domain consists of a private dataset $X$. The dataset $X$ contains 120,000 unique user interactions collected from a live recommender system over the span of 3 weeks for 10,000 users. The recommendation setting is much closer to what practitioners in the field can expect to encounter, with $|\mathcal{A}|=112$ and $\mathcal{X} \subseteq \mathbb{R}^{47}$. The dataset is accompanied by a long-term value estimator $\hat{V}^\mu$
, outputting scores in the interval $[-1,1]$ (where 1 is the best possible score). In this setting, the LTR corresponds to the likelihood of a user to be ``retained'', i.e. re-subscribe to a service. 

\paragraph{Choice of baselines}
\label{section:choice_baselines}
 Our baselines consist of representative recommendation and offline RL methods. We include a standard contextual bandit agent without exploration. Next, we include BCQ~\citep{fujimoto2019off} as the batch RL baseline. BCQ is general enough to be suitable for all target domains, and does not require extensive tuning. We also added two additional baselines that are adapted from the RecoGym repository: (1) bandit with matrix factorization and (2) IPS estimator via neural net regression. The first method, shortened to CB+MF, alternates between a neural matrix factorization of user-product instantaneous proxy rewards via SGD, and selecting the action with highest predicted proxy reward. 
 The second method, shortened to NN-IPS, assumes that data came from a contextual bandit ($T=1$) problem and trains a neural network to approximate the inverse propensity scores (IPS). 
 All experiments report undiscounted test performance on true environment rewards over 200 rollouts and 5 random seeds.

\subsection{Comparison to offline CB and RL}
The first experiment compares the performance of {\method} with that of offline contextual bandits and offline RL.
\begin{figure}[t]
    \centering
    \includegraphics[width=0.5\linewidth]{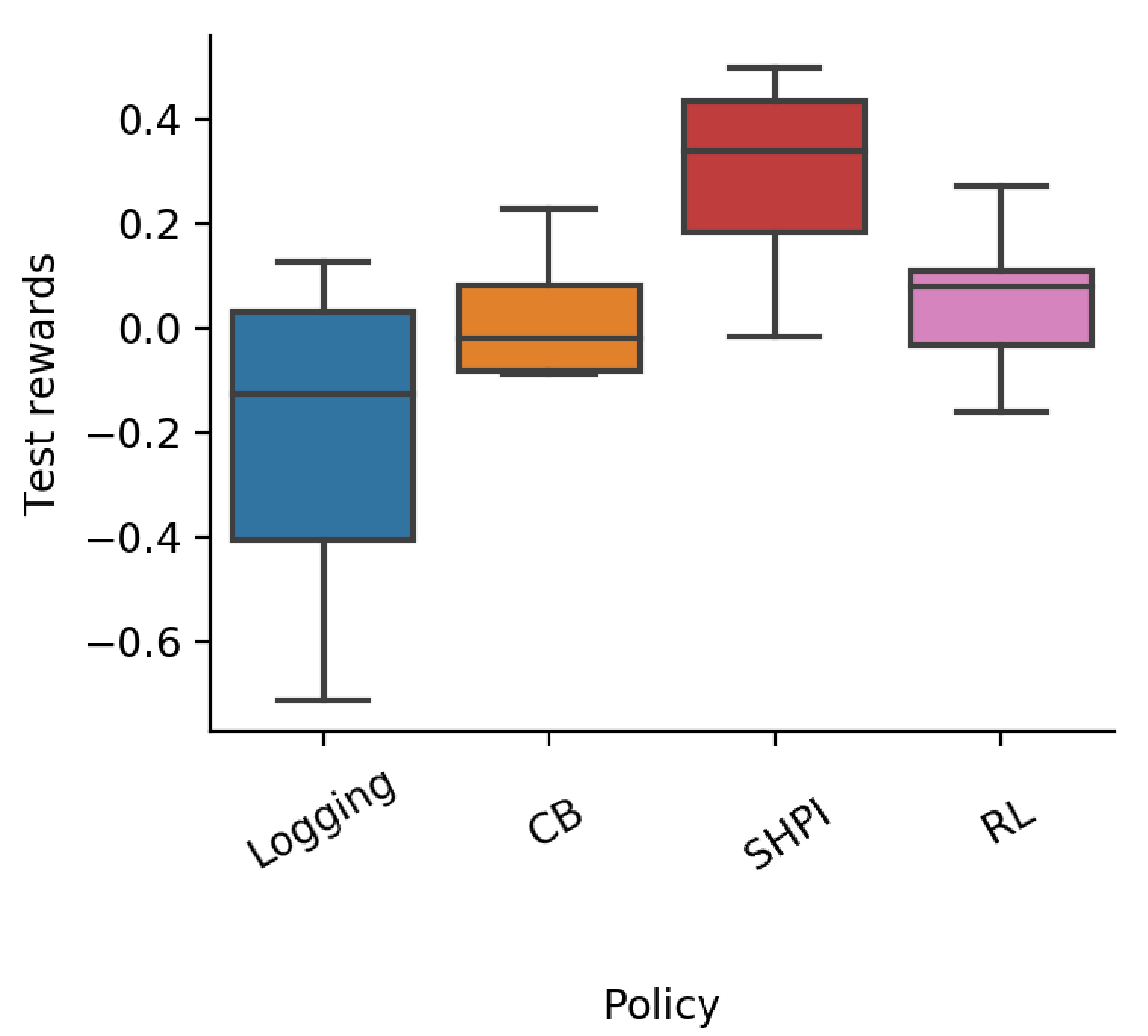}
    \caption{RecoGym online evaluation of algorithms trained on a fixed dataset sampled from the logging policy.}
    \label{fig:boxplots_recogym}
\end{figure}
Our results on the realistic RecoGym domain (Fig.~\ref{fig:boxplots_recogym}) show that {\method} outperforms both offline CB and RL benchmarks in recommendation scenarios. As we show later in Sec.~\ref{sec:expt_k}, the underwhelming performance of offline RL is caused by a large error in estimating $A^\pi$. 

\subsection{Scaling to real-world problems}
In the second experiment, we fitted a model to state transitions and rewards in the private dataset X with a ground-truth long-term value for users. This world model is more realistic than RecoGym: it incorporates a complex LTR function, richer user features and realistic state transitions. 
We ran all algorithms on a simulated batch from this environment logged by an $\varepsilon$-greedy policy and 
report online evaluation results of the resulting policies (Table~\ref{tab:private_eval}).

\begin{table}[t]
\centering
\resizebox{\linewidth}{!}{%
\begin{tabular}{l|lllll}
\toprule
Rewards & Logging &  CB  & SHPI & RL & NN-IPS\\ \midrule
Avg & -0.69 $\pm$ 0.3 & -1.53 $\pm$ 0.03 & -0.29 $\pm$ 0.21 & -1.02 $\pm$ 0.29 & -0.98 $\pm$ 0.04 \\
Normed & 1 & 0.45 & 2.37 & 0.68 & 0.70 \\
\bottomrule
\end{tabular}%
}
\caption{Evaluation on private dataset, higher is better.}
\label{tab:private_eval}
\end{table}

The results above suggest that our method indeed scales to real-world datasets, and performs better than offline CB and offline RL in more complex recommendation scenarios.

\subsection{Using noisy correlated dense training signals}
Throughout the paper, we focused on maximizing user conversions which is a sparse, direct long-term metric of success. In the literature there is an alternative approach which uses noisy but correlated signals (e.g., clicks) instead of the sparse long-term metrics, akin to reward-shaping in RL~\citep{vernadenon}. In this experiment, we compare the performance of recommender systems using LTR rewards with that of systems relying on correlated immediate rewards (we use clicks) within the RecoGym simulator. 
Table~\ref{tab:noisy_ltr} shows evaluation rewards on the RecoGym task, where all algorithms are trained on their respective reward $r$ (click or long-term rewards) but validated on LTR. Note that the LTR row corresponds to Fig.~\ref{fig:boxplots_recogym}.

\begin{table}[t]
\centering
\resizebox{\linewidth}{!}{%
\begin{tabular}{l|llllll}
\toprule
$r$ & Logging &  CB  & SHPI            & RL & Bandit+MF & NN IPS\\ \midrule
Click & 0.05 $\pm$ 0.25 & 0.11 $\pm$ 0.16 & 0.02 $\pm$ 0.23 & -0.03 $\pm$  0.05 & -0.20 $\pm$ 0.21 & 0.11 $\pm$ 0.31  \\
LTR & -0.13 $\pm$ 0.28 & -0.02 $\pm$ 0.18 & 0.33 $\pm$ 0.17  & 0.08 $\pm$ 0.14 & 0.06 $\pm$ 0.19      & 0.25 $\pm$ 0.15\\
\bottomrule
\end{tabular}%
}
\caption{Average online returns obtained by training on short-term vs long-term rewards.}
\label{tab:noisy_ltr}
\end{table}

The MF and IPS baselines are trained to directly maximize the LTR made available to them by the simulator, but without further reasoning beyond the current timestep. 
Under click rewards, we see that all algorithms are taking myopic decisions to maximize immediate rewards and do not optimize the long-term submodular metric as efficiently. 

\subsection{Influence of $k$ on the estimation error} 
\label{sec:expt_k}
We further conduct ablations on the role that $k$ plays in the estimation error of the long-term advantages. To do so, we sample $200$ states $s \sim \mathbb{P}^\mu$ in the synthetic task and perform 30 online rollouts to estimate ground truth advantages for each $s$, and report the mean squared error between offline and online estimates in Fig.~\ref{fig:MSE_adv}.
The results show that using $k=5$ yields lower MSE on average than $\mathbb{A}^\mu$ or $\mathbb{A}^\pi$.  The improved estimation quality also leads to better learned policies on the toy task, as seen in Table~\ref{tab:A_pi_k}. 
\begin{table}[t]
\centering
\begin{tabular}{l||lll}
\toprule
$\varepsilon$ & $\hat{\mathbb{A}}_k$ & $\hat{\mathbb{A}}^\mu$ & $\hat{\mathbb{A}}^\pi$\\ \midrule
 $0.3$ & 8411 $\pm$ 32 & 8097 $\pm$ 34 & 1499 $\pm$ 8005\\
 $0$ & 8346 $\pm$ 180 & 8285 $\pm$ 269 & 4761 $\pm$ 6186\\
\bottomrule
\end{tabular}%
\caption{Average online returns of \method~with $\hat{\mathbb{A}}_k$, $\hat{\mathbb{A}}^\mu$ and $\hat{\mathbb{A}}^\pi$ estimators.}
\label{tab:A_pi_k}
\end{table}

\begin{figure}[t]
    \centering
    \includegraphics[width=0.5\linewidth]{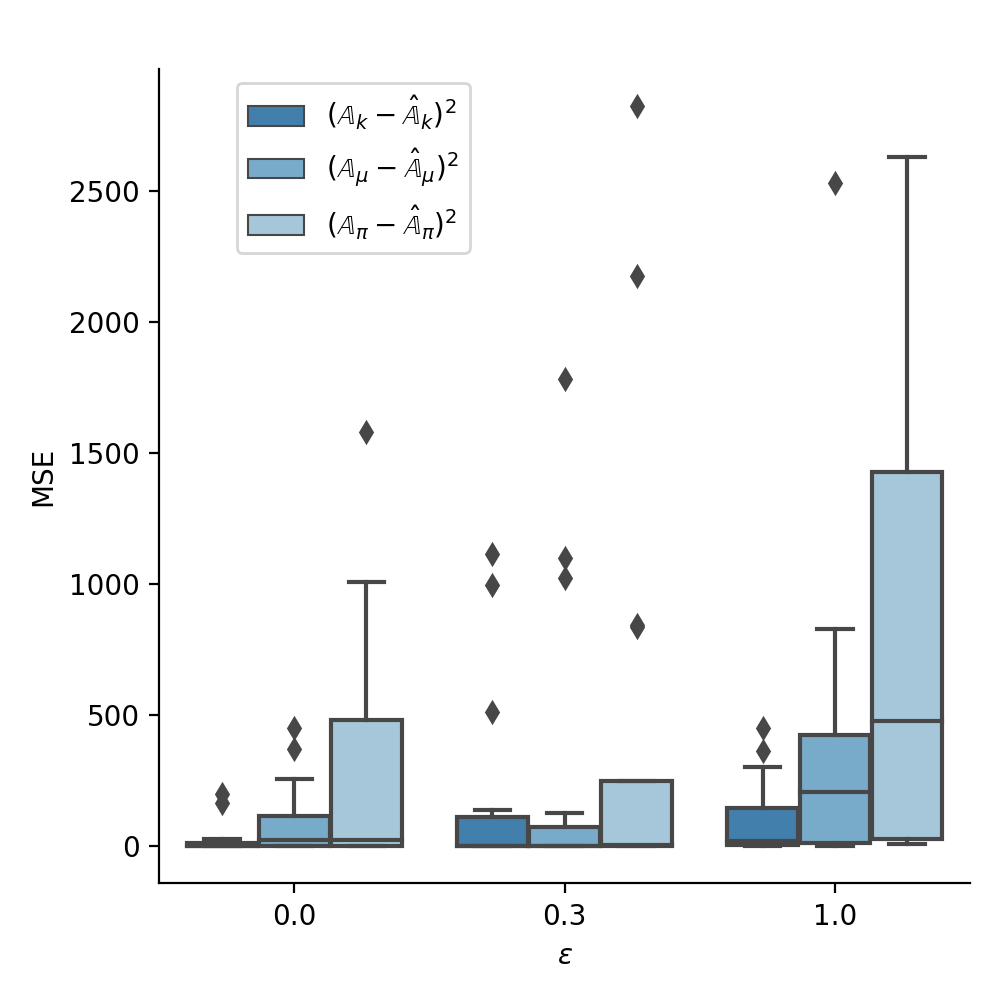}
    \vspace{-1em}
    \caption{Mean-squared error between offline and online estimates of $\mathbb{A}^\pi_{(k)}$, $\mathbb{A}^\mu$, $\mathbb{A}^\pi$ for different choices of $\mu$.}
    \label{fig:MSE_adv}
\end{figure}

\subsection{Robustness to behavior policy}
We validate that {\method} is able to learn from data collected by various behavior policies. To do so, we re-use the online RL critic to create an $\varepsilon$-greedy $\mu$ with $\varepsilon=0.3$ and random uniform ($\varepsilon=1$) policies. Fig.~\ref{fig:boxplots_hiv_behavior} show returns collected by offline and online baselines in the HIV domain, while Table~\ref{tab:boxplots_toy_behavior} in Appendix shows results on the synthetic task. 
Note that {\method} achieves better results when learning from a uniform logging policy as opposed to a policy closer to the optimum. This highlights an instability when learning from an expert policy and suggests that \method~can be composed with complementary off-policy RL techniques like SPIBB and MBS for stable extrapolation. 
In the Appendix, we show experimental results for \method~with one such trust region proximal regularization (as in TRPO~\cite{schulman2015trust}), and we observe more stable learning with non-exploratory datasets.  

\begin{figure}[ht]
    \centering
    \includegraphics[width=0.8\linewidth]{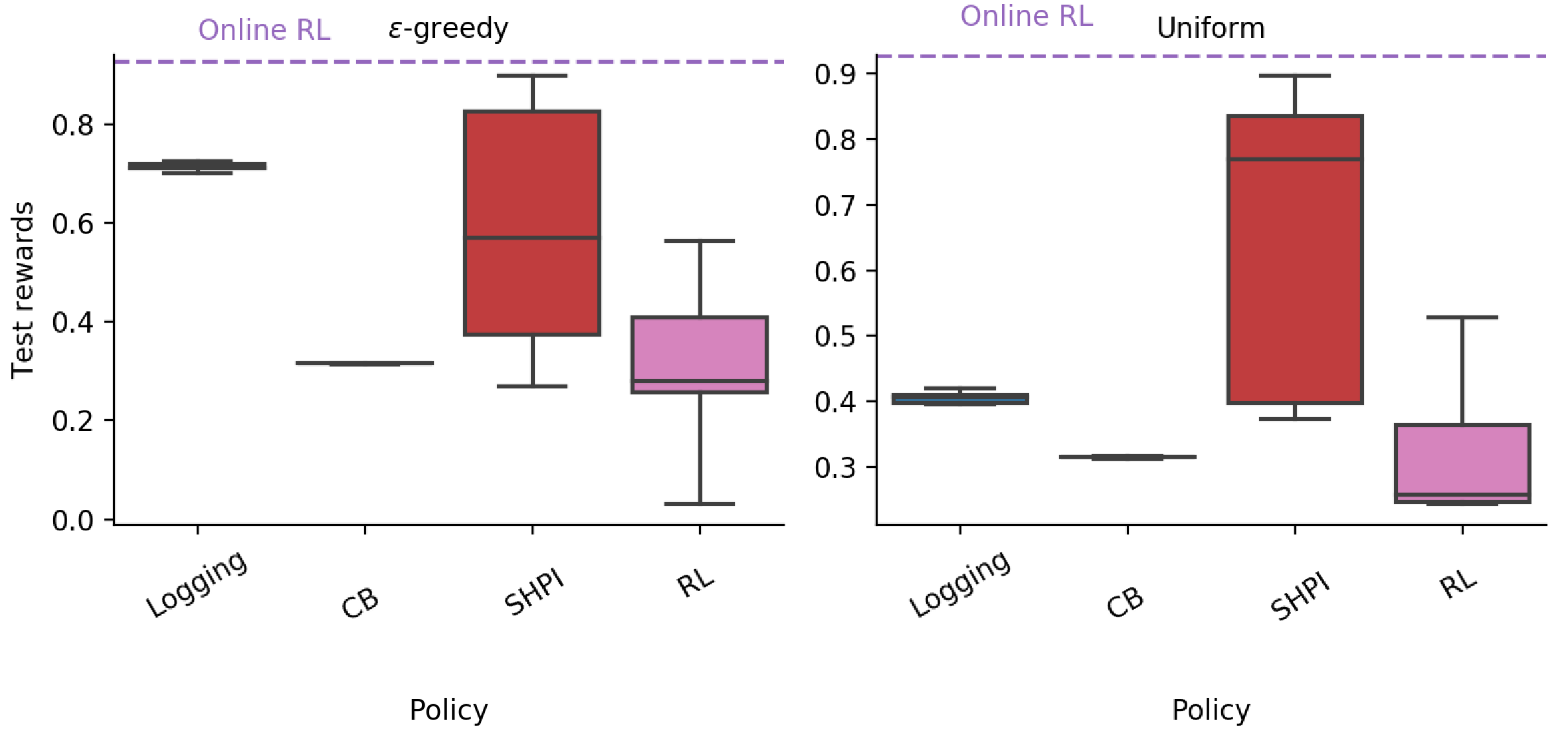}
    \caption{Online evaluation on HIV simulator. }
    \label{fig:boxplots_hiv_behavior}
\end{figure}

%% file: disc.tex
\section{Discussion}
\label{sec:discussion}

We have shown that a suitable termination bonus estimator can be feasibly estimated from batch data and when used along with episodic RL techniques, allows us to witness a policy improvement that other RL advantage estimates can miss. \method~uses these estimated advantages 
to provide batch RL policy improvement. Empirical results show that \method~is particularly effective in improving long-term rewards. 

There are several avenues for further research. 
Can we tune the horizon $k$ in a problem-dependent way? This may allow us to \emph{discover} when session-based or even myopic bandit-based reasoning is sufficient for an application. On a related note, we can develop variants of \method~that incorporate pessimistic $V^*$ or $V^\mu$ estimates for more conservative episodic RL algorithms. The $k$-step advantages used by \method~can also be incorporated in other batch RL algorithms (which predominantly use policy gradient or conservative policy iteration schemes) like SPIBB, MBS, etc. 
Finally, deploying \method~in real-world recommender systems, and studying its performance when user preferences may be non-stationary, remains as future work. 

%% file: appendix.tex
\subsection{Reproducibility Checklist}

We follow the reproducibility checklist~\textit{``The machine learning reproducibility checklist''} and point to relevant sections explaining them here.
\\
For all algorithms presented, check if you include:\\
\begin{itemize}
    \item \textbf{A clear description of the algorithm, see main paper and included codebase.}
    The proposed approach is completely described by Alg.~\ref{alg:online_shpi} and Alg.~\ref{alg:offline_shpi}.
\item \textbf{An analysis of the complexity (time, space, sample size) of the algorithm.}
The complexity of our algorithm depends on the complexity of a regression oracle invoked in each policy iteration step, and benefits from similar convergence properties as approximate policy iteration. In fact, using regression oracles instead of full-blown deep RL can reduce the space complexity in solving the problem.

\item \textbf{A link to a downloadable source code, including all dependencies.} The code is included with the supplementary material as a folder; all dependencies can be installed using Python's package manager.

\end{itemize}
For all figures and tables that present empirical results, check if you include:
\begin{itemize}
    \item \textbf{A complete description of the data collection process, including sample size.} We use standard benchmarks provided in OpenAI Gym (Brockman et al., 2016). Additional dataset details are provided in Sec.~\ref{app:expt}. 
    \item \textbf{A link to downloadable version of the dataset or simulation environment.} Not applicable.
\item \textbf{An explanation of how samples were allocated for training / validation / testing.} We do not use a training-validation-test split, but instead report the mean performance (and one standard deviation) of the policy at evaluation time across 10 trials.
\item \textbf{An explanation of any data that were excluded.} 
Not applicable.
\item \textbf{The exact number of evaluation runs.} 10 trials to obtain all figures, 200 online rollouts each.
\item \textbf{A description of how experiments were run.} See Experiment and Appendix sections.
\item \textbf{A clear definition of the specific measure or statistics used to report results.} Undiscounted returns across the whole episode are reported and aggregated as follows: $Y_T-Y_0$ (synthetic), $T^{-1}\sum_{t=1}^TY_t$ (RecoGym) and $\sum_{t=1}^TY_t$ (HIV). 
\item \textbf{Clearly defined error bars.} Confidence intervals are always mean$\pm\;1$ standard deviation over 10 trials.
\item \textbf{ A description of results with central tendency (e.g. mean) and variation (e.g. stddev)}. All table results use the mean and standard deviation. Boxplots use the standard quantile delimitations.
\item \textbf{ A description of the computing infrastructure used.} All runs used 1 CPU and 1 P40 GPU with $8$ Gb of memory. 
\end{itemize}

\newpage

\input{state_distribution_matching}

\subsection{Experiments}
\label{app:expt}

\paragraph{Dataset details.} When a recommender interacts with the environment, we get a single stream of (context, action, feedback) data. To break this into $T$-length episodes, we adopt the following setup for all offline agents: first, a set of trajectories is collected using logging policy $\mu$ and starting state $x_0 \sim p_0$. Then, a sliding window of length $T$ is applied every $\delta$ steps, e.g. $X_{1:T},X_{1+\delta:T+\delta},..$. This new windowed batch is fed to all offline algorithms. Note that when $\delta=1$, the algorithm sees each datapoint $T$ times, while $\delta=T$ leads to weakly dependent trajectories, but is inefficient as it throws away $T-1$ instances of each datapoint. We picked a reasonable value for $T$ based on the reward functions in real-world scenarios (e.g., the private dataset uses 4 weeks of past data to define rewards and hence $T=28$); then $\delta$ is picked to be in $(1,T)$.

\paragraph{Four Recommendation Domains} The first benchmark implements a real-world scenario in which $|\mathcal{A}|$ randomly sampled action vectors additively affect the states of a user. Action effects are distilled in time as the current state is an average of previously seen contexts; and the reward function is characterized by the Styblinski-Tang function~\citep{styblinski1990experiments}. The domain is described fully in Sec.~\ref{sec:synth_domain}. 
The second benchmark, RecoGym, is a personalized recommendation domain commonly used to study recommender systems based on RL reasoning. We modified its reward signal to be a submodular function dependent on all previous products, described in Sec.~\ref{sec:synth_recogym}. The last benchmark is the HIV treatment simulator with a small state and action spaces, but representative of real-world recommendation in the healthcare domain.
\subsection{Synthetic recommendation task}
\label{sec:synth_domain}
The dynamics updates are as follow:
\begin{equation}
 \begin{cases}
 \mu_0 = [0,0,..,-10,..,0]\in \mathbb{R}^d\\
 a_{1:\tau}\sim \pi(\cdot|W_{1}),..,\pi(\cdot|W_\tau)\\
 W_{1:\tau} \sim \mathcal{N}(\mu_0,\mat{I}_d)\\
 W_{t+1}=\frac{1}{\tau}\sum_{i=0}^{\tau-1} (W_{t-i}+\vec{a}_{{t-i}})\\
 X_{t+1}=\frac{1}{\rho}\sum_{i=0}^{\rho-1} W_{t-i+1}
 \end{cases}
\end{equation}
 where $a$ is the action vector and $\mu_0$ is a vector of dimension $d$ with all elements equal to 0 except one, which is set to -10 uniformly at random.
 
 \paragraph{Why is the toy environment representative of real-world recommendations?} While this toy setup is characterized by a simple MDP, it has three important properties found in typical user interaction datasets:
\begin{enumerate}
    \item The contexts are slow-varying count features, which is simulated by the averaging of process $\{W_t\}_{t=1}^T$ into $\{X_t\}_{t=1}^T$. 
    \item Every action taken has a distributed effect on the context, since the current action is averaged with $\tau$ previous actions.
    \item The reward signal is emitted based on user features encoded in the context, and the dependence on the action is implicit through $X$. This is a common setup for tasks where offline data can be learned via a supervised model, but contributions of individual actions cannot be isolated.
\end{enumerate}

\begin{table}[ht]
\centering
\begin{tabular}{l||llll}
\toprule
$\varepsilon$ & $\mu$ & {\method} & Off.RL & On.RL\\ \midrule
 $0.3$ & 5125 $\pm$ 316 & 8411 $\pm$ 32 & 8233 $\pm$ 271 & 8144 $\pm$ 311\\
 $1$ & 77 $\pm$ 402 & 8346 $\pm$ 180 & 2669 $\pm$ 3085 & 8144 $\pm$ 311\\
\bottomrule
\end{tabular}%
\caption{Online evaluation on toy synthetic task.}
\label{tab:boxplots_toy_behavior}
\end{table}

\paragraph{Larger-scale simulation}
\begin{figure}[ht!]
    \centering
    \includegraphics
    {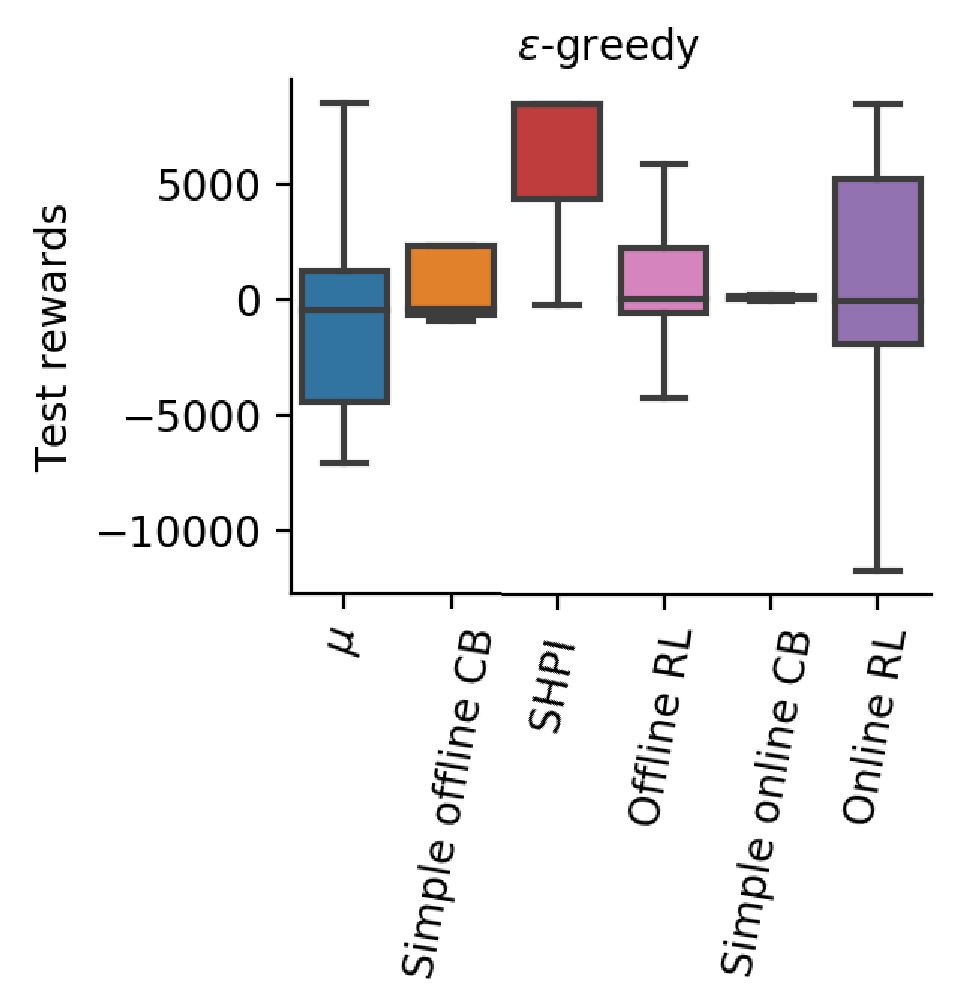}
    \caption{Test results on a larger synthetic benchmark}
    \label{fig:toy_large}
\end{figure}

The results presented in the main paper are for $d=2$ and $|\mathcal{A}|=10$. Figure~\ref{fig:toy_large} shows test (online) results on a larger benchmark ($d=100$ and $|\mathcal{A}|=300$) with $\lambda=10$ and $\mu$ being $\varepsilon$-greedy with $\varepsilon=0.3$. Note how BCQ does not manage to reach a decent performance, due to a short reasoning horizon. It's worth noting that the contexts are autoregressive up to $T=28$, and therefore many RL methods would fail to isolate the action which lead to a specific change in contexts. We additionally included the online CB and RL baselines to highlight how a high-dimensional action/state space can affect their convergence.

\subsection{RecoGym}
\label{sec:synth_recogym}
We modify the official code of RecoGym found here: \url{https://github.com/criteo-research/reco-gym} to include a sub-modular reward function obtained by filtering all items that have a click (i.e. $X_t$ s.t. $r_t=1$), applying component-wise maximum and multiplying with the $\Omega$ affinity vector described in \citet{rohde2018recogym}.



\subsubsection{HIV}
The environment and default parameters are taken from the repo \url{https://github.com/StanfordAI4HI/RepBM/tree/master/hiv_domain/hiv_simulator}. All results on HIV report the cumulative episode reward $\sum_{t=1}^TY_t$. No further modifications were done.

\subsection{\method~without rewards: Using pre-trained $V^\mu$}
In some real-world scenarios, we don't have access to a training dataset (e.g., for privacy compliance), but instead only have access to a $V^\mu$ model trained on historical data. How can we apply SHPI without having access to tuples with explicit $r(x,a)$ observations \emph{but} instead being able to evaluate $V^\mu$ at any state? 
The temporal difference equation shows that 
$$r(s,a) \approx V^\mu(s) - \gamma \mathbb{E}_{s'\sim\mathcal{T}(\cdot|s,a)}V^\mu(s'),$$
and so a noisy but unbiased estimate on a tuple $(s,a,s')$ is $\hat{r}(s,a) = \hat{V}^\mu(s) - \gamma \hat{V}^\mu(s')$. We call this estimation procedure \emph{backshift}, in connection with a similar approach in econometrics literature.

Table~\ref{tab:toy_oracle_bias} shows that performance of simple (non-backshift) algorithms suffers from using $r(s,a)=V^\mu(s)$, while backshift {\method} still performs well. To stress-test the backshift approach, we also experimented with corrupted rewards in the synthetic task. A non-stationary bias was sampled every timestep from $\mathcal{N}(10^5,10^4)$ and added to the output of the reward oracle, such that the effective range of outputs for that scenario was $10^4+(10^5+10^4)$ with probability $0.68$ using empirical rule (i.e., $68\%$ of the data is within $80\%$ of the mean). Although non-stationary, the bias was set to be slowly varying (every 10 timesteps), which allowed the backshift estimators to treat it as locally-stationary and remove it. The bias was added at the data collection phase only; online test results were performed with the true, unbiased reward function. The session-based metric that is reported in all plots for this task is $Y_T-Y_0$, i.e. the improvement observed during this episode with respect to the start state.

\begin{table}[ht]
\centering
\begin{tabular}{l||lll}
\toprule
Task & S.{\method} & B.{\method} \\ \midrule
Unbiased reward  & 8077 $\pm$ 992 & 8350 $\pm$ 183\\
Biased reward  &  -557 $\pm$ 429 & 5714 $\pm$ 364\\
\bottomrule
\end{tabular}%
\caption{Online evaluation on synthetic task of simple (S.) and backshift (B.) estimators.}
\label{tab:toy_oracle_bias}
\end{table}

In the synthetic task, the LTR oracle was explicitly simulated. However, situations where the reward signal is noisy can arise in real-world scenarios as well. For example, the HIV simulator simulates its reward by discretizing a system of differential equations, which can introduce additional noise~\citep{ernst2006clinical}. 
Fig.~\ref{fig:boxplots_hiv_oracle_bias} validates that the backshift + {\method} coupling is indeed helpful in the HIV task.

\begin{figure}[ht]
    \centering
    \includegraphics
    {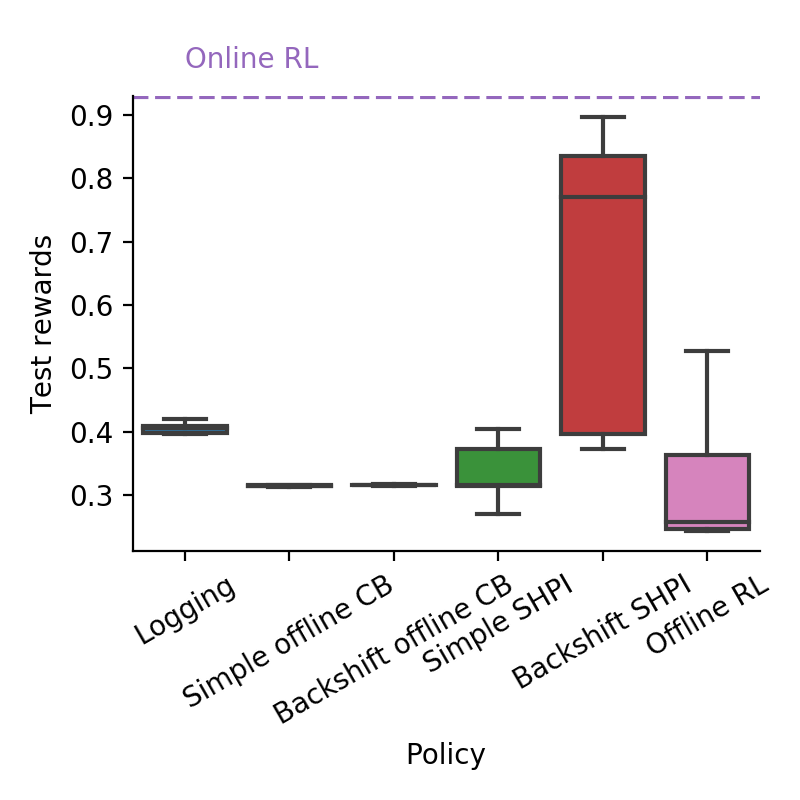}
    \caption{Online evaluation on HIV simulator. }
    \label{fig:boxplots_hiv_oracle_bias}
\end{figure}

\subsection{Experimental parameters}
The regression oracle at the core of our algorithm is implemented using the well-known VowpalWabbit (VW) library in Python: \url{https://vowpalwabbit.org/}.
\begin{table}[ht!]
    \centering
\begin{tabular}{c|c |c| c}
\toprule
Hyperparameters    & Synthetic & RecoGym & HIV  \\ 
\midrule
$d$   & [2,100]     & 100 &  6  \\
$|\mathcal{A}|$              & [10,300]    & 100 &     4  \\ 
$T$ (Interaction horizon) & 150 & 200 & 200\\
$\delta$ (Sliding window step) & $20$ & $10$ & $40$\\
$\rho$ (Sliding window size) & $30$ & $50$ & $50$\\
$n_{offline}$ & 2,000 & 2,000 & 5,000\\
$n_{online}$ & 1,000 & 2,000 & 10,000\\
\hline
$k$ & \multicolumn{3}{c}{5}  \\
VW & \multicolumn{3}{c}{\texttt{---coin ---cb\_type mtr ---cb\_explore\_adf }}\\
DualDice & \multicolumn{3}{c}{Defaults from \url{https://github.com/google-research/dice_rl}}  \\
$\gamma$ & \multicolumn{3}{c}{0.99}\\  
Optimizer      & \multicolumn{3}{c}{Adam}   \\ 
Learning rate      & \multicolumn{3}{c}{$10^{-3}$}   \\ 
$V^\mu$, BCQ hidden layers        & \multicolumn{3}{c}{$128 \times 128$}    \\
Random seeds       & \multicolumn{3}{c}{0 to 9} \\   
\bottomrule
\end{tabular}
    \caption{Hyperparameters used across experiments. Top part is environment specific; bottom part is algorithm-specific and constant for all domains.}
    \label{tab:hyperparams}
\end{table}







%% file: state_distribution_matching.tex
\subsection{Policy improvement guarantees}
\label{sec:finite}

The following theorem connects the advantage estimate maximized by Algorithm~\ref{alg:offline_shpi} to the performance difference between $\mu$ and $\pi$. This is a straightforward generalization of the performance difference lemma~\citep{kakade2002approximately}. 
\begin{theorem}
Let $\mu$ and $\pi$ be two arbitrary policies acting in a $T$-horizon episodic MDP and let $1 \le k \le T$. Then, 
\begin{equation*}
    \begin{split}
       \sum_{t=1}^{T-1}\mathbb{E}_{\mathbb{P}^{\mu}_t}[ V^\pi_{t}(x_{t})-V^\mu_{t}(x_{t}) ]  = \sum_{t=0}^{T-1}\mathbb{E}_{\mathbb{P}^{\mu}_t}[\mathbb{A}^{\pi}_{(k),t}(x_t,a_t)] + \\
        \gamma^k \mathbb{E}_{\mathbb{P}^{\mu}_{t}} \mathbb{E}_{\mathbb{P}^{\pi}_{t:t+k}}[ (V^\pi_{t+k}(x_{t+k})-V^\mu_{t+k}(x_{t+k}))]. 
    \end{split}
\end{equation*}
\label{thm:pdl}
\end{theorem}

\begin{proof}
The following statement of the theorem can be obtained by re-arranging the terms from the classical PDL of~\citet{kakade2002approximately}:
\begin{equation*}
    \begin{split}
        \eta^\mu - \eta^\pi = \sum_{t=0}^{T-1}\mathbb{E}_{\mathbb{P}^{\mu}_t}[\mathbb{A}^{\pi}_{(k),t}(x_t,a_t)] +
        \mathbb{E}_{\mathbb{P}^{\mu}_t}[ V^\mu_{t}(x_{t})-V^\pi_{t}(x_{t}) ] + 
        \gamma^k \mathbb{E}_{\mathbb{P}^{\mu}_{t}} \mathbb{E}_{\mathbb{P}^{\pi}_{t:t+k}}[ (V^\pi_{t+k}(x_{t+k})-V^\mu_{t+k}(x_{t+k}))]. 
    \end{split}
\end{equation*}
We begin by expanding the right-hand side of the performance difference from the $T$ horizon expression (which can be found in~\citet{kakade2003sample}):
\begin{equation*}
    \begin{split}
        \eta^\mu - \eta^\pi &= \sum_{t=0}^{T-1}\mathbb{E}_{\mathbb{P}^{\mu}_{t}}[\mathbb{A}^{\pi}_{t}(x_t,a_t)]\\
        &=\sum_{t=0}^{T-1}\mathbb{E}_{\mathbb{P}^\mu_t}[\mathbb{E}_{\mathbb{P}^\pi_{t:t+k} }[\sum_{m=0}^{k-1}\gamma^mr_{t+m}
     +\gamma^kV^\pi_{t+k}(X_{t+k})|x_t,a_t]-V^\pi_t(x_t)]+\mathbb{E}_{\mathbb{P}^\mu_t}[\mathbb{E}_{\mathbb{P}^\pi_{t:t+k} }[\gamma^kV^\mu_{t+k}(X_{t+k})-\gamma^kV^\mu_{t+k}(X_{t+k})]\\
     &+\mathbb{E}_{\mathbb{P}^{\mu}_t}[ V^\mu_{t}(x_{t})-V^\mu_{t}(x_{t}) ]\\
     &=\sum_{t=0}^{T-1}\mathbb{E}_{\mathbb{P}^{\mu}_t}[\mathbb{A}^{\pi}_{(k),t}(x_t,a_t)] +
        \mathbb{E}_{\mathbb{P}^{\mu}_t}[ V^\mu_{t}(x_{t})-V^\pi_{t}(x_{t}) ] + 
        \gamma^k \mathbb{E}_{\mathbb{P}^{\mu}_{t}} \mathbb{E}_{\mathbb{P}^{\pi}_{t:t+k}}[ (V^\pi_{t+k}(x_{t+k})-V^\mu_{t+k}(x_{t+k}))]
    \end{split}
\end{equation*}
Putting the second term from the RHS to the LHS yields the result.

\end{proof}

\method~optimizes the first term in the RHS by greedily maximizing $\mathbb{A}^\pi_{(k),t}$ for every $t$. 
The LHS is akin to a cumulative performance difference over $T$ timesteps. 
If $k=1$ and we only optimize the first term on the RHS, 
ignoring the other terms can be a severe under-estimate of the true performance difference. 
However if $k=T$, then 
the first term is equivalent to optimizing $\mathbb{A}^\pi_t$ and the second term is equal to $0$ (since $V^\pi_t=0$ for all $t \ge T$). 
So, \method~with $k=1$ is a policy gradient step, while with $k=T$ it is maximizing the performance difference.

\begin{lemma}
Suppose $\hat{V}^\mu$ and $\hat{A}^\pi_{(k)}$ are estimated on separate independent datasets. Then
\begin{equation}
    \begin{split}
        \mathbb{E}_{\mathbb{P}^\mu_{t+1:t+k}}&[(\hat{\mathbb{A}}^\pi_{(k)}(t,n)-\mathbb{A}^\pi_{(k)}(t))^2]\leq \\ &\mathbb{E}_{\mathbb{P}^\mu_{t+1:t+k+1}}[(\hat{\mathbb{A}}^\pi_{(k+1)}(t,n)-\mathbb{A}^\pi_{(k+1)}(t))^2]
    \end{split}
\end{equation}
\label{thm:finite_sample_bias}
\end{lemma}
The proof is by direct induction on $k$. 
We can actually provide a precise characterization of how errors in estimating $\hat{w}$, $\hat{r}$ and $\hat{V}^\mu$ affect the estimation quality of $\mathbb{A}^\pi_{(k)}$ in Lemma~\ref{lemma:bias_lemma}. 
Applying induction w.r.t. $k$ in that lemma also shows the desired result. 
Lemma~\ref{thm:finite_sample_bias} shows that we should not set $k$ to be too large lest we have overwhelming estimation errors in the advantages. 
In practice we set $k$ equal to the typical user session length. 

\begin{lemma}
Let $\hat{\mathbb{A}}^\pi_{(k)}(t)$ be defined as in Eq.~\ref{eq:advantage_MC}, $\hat{w}\in [q_1,q_2]$ and $r_{max}=\sup_{x,a} r(x,a)$. Then,
\begin{equation*}
    \begin{split}
          |\mathbb{E}_{\mathbb{P}^\mu_{t+1:t+k}}[\hat{\mathbb{A}}^\pi_{(k)}(t,n)]-\mathbb{A}^\pi_{(k)}(t)|&\leq\frac{1-\gamma^k}{1-\gamma}q_2\bigvee_{m=0}^{ k-1}\epsilon_r^{t+m}(n)+\frac{1-\gamma^k+\gamma^{t+k}-\gamma^{T+1}}{1-\gamma}\big|r_{max}\big|\bigvee_{m=0}^{k-1}\epsilon_w^{t+m}(n)\\
          &+\gamma^kq_2\epsilon_V^{t+k}(n)+\epsilon_V^t(n)\\
    \end{split}
    \label{eq:bias_lemma}
\end{equation*}
\label{lemma:bias_lemma}
\end{lemma}

\begin{proof}
\begin{equation}
   \begin{split}
        \mathbb{E}_{\mathbb{P}^\mu_{t+1:t+k}}[\hat{\mathbb{A}}^\pi_{(k)}(t,n)]=\mathbb{E}_{\mathbb{P}^\mu_{t+1:t+k-1}}[\sum_{m=0}^{k-1}\gamma^m \frac{1}{n}\sum_{i=1}^nw_{i,t+1}^{t+m}r_{i,t+m}]+\gamma^k\mathbb{E}_{\mathbb{P}^\mu_{t+1:t+k}}[\frac{1}{n}\sum_{i=1}^n\hat{w}_{i,t+1}^{t+k}\hat{V}^\mu(x_{i,t+k})]-\hat{V}^{\mu}(x_t)\;.
   \end{split}
\end{equation}

Then,
\begin{equation*}
    \begin{split}
        &|\mathbb{E}_{\mathbb{P}^\mu_{t+1:t+k}}[\hat{\mathbb{A}}^\pi_{(k),t}]-\mathbb{A}^\pi_{(k),t}|\\
        &=|\sum_{m=0}^{k-1}\gamma^m\frac{1}{n}\sum_{i=1}^n\mathbb{E}_{\mathbb{P}^\mu_{t+1:t+m}}[\hat{w}_{i,t+1}^{t+m}\hat{r}_{i,t+m}-w_{t+1}^{t+m}r_{t+m}]+\gamma^k\frac{1}{n}\sum_{i=1}^n\mathbb{E}_{\mathbb{P}^\mu_{t+1:t+k}}[\hat{w}_{i,t+1}^{t+k}\hat{V}^\mu_{i,t+k}-w_{t+1}^{t+k}V^\mu_{t+k}]+V^\mu_t-\hat{V}^\mu_t|\\
        &\leq\sum_{m=0}^{k-1}\gamma^m\frac{1}{n}\sum_{i=1}^n\big|\mathbb{E}_{\mathbb{P}^\mu_{t+1:t+m}}[\hat{w}_{i,t+1}^{t+m}\hat{r}_{i,t+m}-w_{t+1}^{t+m}r_{t+m}]\big|+\gamma^k\frac{1}{n}\sum_{i=1}^n\mathbb{E}_{\mathbb{P}^\mu_{t+1:t+k}}[|\hat{w}_{i,t+1}^{t+k}\hat{V}^\mu_{i,t+k}-w_{t+1}^{t+k}V^\mu_{t+k}|]+|V^\mu_t-\hat{V}^\mu_t|\\
    \end{split}
\end{equation*}
In this case,
$$
\frac{1}{n}\sum_{i=1}^n\mathbb{E}_{\mathbb{P}^\mu_{t+1:t+k-1}}[\hat{w}_{i,t+1}^{t+m}]\neq \mathbb{E}_{\mathbb{P}^\mu_{t+1:t+k-1}}[w_{t+1}^{t+m}]
$$
for arbitrary $(x_t,a_t)$, as a thorough coverage of the space is not guaranteed and it's impossible to uniformly control the error. It is, however, possible to bound the error by first defining the difference functional as
$$
\Delta_t^T(f)=f(a_{t:T},x_{t:T})-\hat{f}(a_{t:T},x_{t:T})
$$

Then,
\begin{equation*}
    \begin{split}
        \big|\mathbb{E}_{\mathbb{P}^\mu_{t+m}}[\hat{w}_{i,t+1}^{t+m}\hat{r}_{i,t+m}-w_{t+1}^{t+m}r_{t+m}]\big|&\leq \big|\mathbb{E}_{\mathbb{P}^\mu_{t+m}}[\hat{w}_{i,t+1}^{t+m}]\mathbb{E}_{\mathbb{P}^\mu_{t+m}}[\Delta_{t+m}(r)]\big|+\big|\mathbb{E}_{\mathbb{P}^\mu_{t+m}}[r_{t+m}]\mathbb{E}_{\mathbb{P}^\mu_{t+m}}[-\Delta_{t+1}^{t+m}(w)]\big|\\
        &\leq\big|\mathbb{E}_{\mathbb{P}^\mu_{t+m}}[\hat{w}_{i,t+1}^{t+m}]\big|\big|\mathbb{E}_{\mathbb{P}^\mu_{t+m}}[\Delta_{t+m}(r)]\big|+\mathbb{E}_{\mathbb{P}^\mu_{t+m}}[\big|r_{t+m}\big|]\big|\mathbb{E}_{\mathbb{P}^\mu_{t+m}}[-\Delta_{t+1}^{t+m}(w)]\big|\\
        &\leq \big|\mathbb{E}_{\mathbb{P}^\mu_{t+m}}[\hat{w}_{i,t+1}^{t+m}]\big|\epsilon_r^{t+m}+\big|r_{max}\big|\epsilon_w^{t+m}
    \end{split}
\end{equation*}

Replacing all differences between estimated and true values of $w,r$ and $V^\mu$, we get
\begin{equation*}
    \begin{split}
        &|\mathbb{E}_{\mathbb{P}^\mu_{t+1:t+k}}[\hat{\mathbb{A}}^\pi_{(k),t}]-\mathbb{A}^\pi_{(k),t}|\\
      &\leq\sum_{m=0}^{k-1}\gamma^m\frac{1}{n}\sum_{i=1}^n\big|\mathbb{E}_{\mathbb{P}^\mu_{t+1:t+m}}[\hat{w}_{i,t+1}^{t+m}]\big|\epsilon_r^{t+m}+|\mathbb{E}_{\mathbb{P}^\mu_{t+m}}[r_{t+m}]|\epsilon_w^{t+m}+\gamma^k\frac{1}{n}\sum_{i=1}^n\big|\mathbb{E}_{\mathbb{P}^\mu_{t+k}}[\hat{w}_{i,t+1}^{t+k}]\big|\epsilon_V^{t+k}+\big|\mathbb{E}_{\mathbb{P}^\mu_{t+k}}[V^\mu_{t+k}]\big|\epsilon_w^{t+k}+\epsilon_V^t\\
      &=\sum_{m=0}^{k-1}\bigg(\gamma^m\big|\mathbb{E}_{\mathbb{P}^\mu_{t+1:t+m}}[\hat{w}_{t+1}^{t+m}]\big|\epsilon_r^{t+m}+|r_{max}|\epsilon_w^{t+m}\bigg)+\gamma^k\big|\mathbb{E}_{\mathbb{P}^\mu_{t+k}}[\hat{w}_{t+1}^{t+k}]\big|\epsilon_V^{t+k}+\big|\mathbb{E}_{\mathbb{P}^\mu_{t+k}}[V^\mu_{t+k}]\big|\epsilon_w^{t+k}+\epsilon_V^t\\
      &=\sum_{m=0}^{k-1}\gamma^m\bigg(\big|\mathbb{E}_{\mathbb{P}^\mu_{t+1:t+m}}[\hat{w}_{t+1}^{t+m}]\big|\epsilon_r^{t+m}+|r_{max}|\epsilon_w^{t+m}\bigg)+\gamma^k\big|\mathbb{E}_{\mathbb{P}^\mu_{t+k}}[\hat{w}_{t+1}^{t+k}]\big|\epsilon_V^{t+k}+\frac{\gamma^{t+k}-\gamma^{T+1}}{1-\gamma}\big|r_{max}\big|\epsilon_w^{t+k}+\epsilon_V^t
    \end{split}
\end{equation*}

If the estimators $\hat{w}$ are clipped to the range $[q_1,q_2]$ as is commonly the case, the expression simplifies to 
\begin{equation}
    \begin{split}
        &|\mathbb{E}_{\mathbb{P}^\mu_{t+1:t+k}}[\hat{\mathbb{A}}^\pi_{(k),t}]-\mathbb{A}^\pi_{(k),t}|\\
        &\leq\sum_{m=0}^{k-1}\gamma^mq_2\max_{0\leq m\leq k-1}\epsilon_r^{t+m}+\frac{1-\gamma^k+\gamma^{t+k}-\gamma^{T+1}}{1-\gamma}\big|r_{max}\big|\max_{0\leq m\leq k}\epsilon_w^{t+m}+\gamma^k q_2\epsilon_V^{t+k}+\epsilon_V^t\\
        &\leq\frac{1-\gamma^k}{1-\gamma}q_2\max_{0\leq m\leq k-1}\epsilon_r^{t+m}+\frac{1-\gamma^k+\gamma^{t+k}-\gamma^{T+1}}{1-\gamma}\big|r_{max}\big|\max_{0\leq m\leq k-1}\epsilon_w^{t+m}+\gamma^kq_2\epsilon_V^{t+k}+\epsilon_V^t\\
    \end{split}
\end{equation}
\end{proof}


\subsection{Conservative updates via proximal regularization}
\label{sec:trust_region}
Policy gradient 
methods are known to be unstable due to estimation errors in $A^\mu$ over regions that are poorly covered by $\mu$.
There are many different approaches in the literature to introduce conservatism in policy updates to remain robust to this failure mode.
One class of approaches introduces a trust region or a proximal regularization of policy updates towards $\mu$ to enforce conservatism (CPI~\citep{kakade2002approximately}, TRPO~\citep{schulman2015trust}, PPO~\citep{schulman2017proximal}, etc.). Another class of approaches introduce pessimism in the estimated advantages to ensure safe policy updates (SPIBB~\citep{laroche2019safe}, CQL~\citep{kumar2020conservative}, MBS~\citep{liu2020provably}, etc.). We will now discuss a notion of trust region that is weaker than those studied in the literature, and show how it can be incorporated into \method~using penalized advantage estimates. This replaces the greedy updates of \method~to ensure that updated policies are still well-supported by data from $\mu$.

A common trust region/proximal regularization constrains $\{\pi:\max_x  KL(\mu(\cdot|x) || \pi(\cdot|x)) < \epsilon\}$. 
With very small $\epsilon$ only $\mu$ remains in this policy set. Remember we want to find up to $k$-step deviations from $\mu$ which may require a large $KL(\mu(\cdot|x) || \pi(\cdot|x))$ for some $x$. 
So we relax the proximal constraint to instead reason about only the state distributions induced by $\mu$ and $\pi$. Let $d^\mu = \frac{1}{T}\sum_t \mathbb{P}^\mu_t$. We introduce a constraint $\{\pi: KL(d^\mu || d^\pi) < \epsilon\}$. Note now that even with very small $\epsilon$, there can be several candidate policies that are allowed in this set. For an extreme case consider a contextual bandit ($T=1, \epsilon \rightarrow 0$); this \emph{state-based} trust region admits all stationary policies while the typical \emph{action-constrained} trust region only admits $\mu$. 
The Lagrangian relaxation of this constraint is:
\begin{equation}
\begin{split}
\argmax_\pi \mathbb{E}_{d^\mu} [\mathbb{A}^\pi_{(k)}(x,a)] \text{ s.t. } KL(d^\mu || d^\pi) \le \epsilon.\\
\equiv \argmax_\pi \frac{1}{n}\sum_{i=1}^n \hat{\mathbb{A}}^\pi_{(k)}(x_i,a_i) - \Lambda \log w_{\frac{\pi}{\mu}}(x_i)
\end{split}
\label{eq:proximal_regularization}
\end{equation}
for $x_i\sim d^\mu, a_i \sim \mu(x_i)$ and setting $w_{\frac{\pi}{\mu}}=\frac{d^\pi}{d^\mu}$. Incorporating these penalized advantages into a regression oracle, we get,
\begin{equation}
    \hat{f}(\Lambda)=\argmin_{f\in\mathcal{F}}\sum_{(x,a,\hat{\mathbb{A}}^\pi_{(k)},w_{\frac{\pi}{\mu}})}(f(x,a)-(\hat{\mathbb{A}}^\pi_{(k)}-\Lambda \log w_{\frac{\pi}{\mu}}))^2,
    \label{eq:proximal_regularization_oracle}
\end{equation}
where $\Lambda>0$ is a parameter controlling the importance of what we call the \emph{distribution matching} constraint.
Let $\pi(\Lambda)$ be the policy induced by $\hat{f}(\Lambda)$ for $\Lambda\geq 0$. 
Setting $\Lambda=0$ recovers 
the $\hat{f}$ found by Online/Offline~\method~ (Alg.~\ref{alg:online_shpi},\ref{alg:offline_shpi}). 
If $\Lambda \gg 0$ then the distribution matching constraint dominates and $\pi(\Lambda)$ will be the uniform random policy. For a well-tuned value of $\Lambda$ we recover policies $\pi(\Lambda)$ close to $\mu$, i.e., $KL(d^\mu||d^\pi)$ is small. 

Finally, we remark on a chicken-and-egg problem in successfully incorporating the proximal regularization. $w_{\frac{\pi}{\mu}}$ is unknown and must be estimated from data, using techniques like DualDICE~\citep{nachum2019dualdice}.
However, these techniques can fail when $d^\pi$ is not well-supported by $d^\mu$.
We introduce an assumption that avoids this difficulty: 
\begin{assumption}
\label{ass:coverage}
There exists $\epsilon>0$ s.t. for the initial policy $\pi_0$:
$$
\max_{x\in\mathcal{X}: d^{\pi_0}(x)>0}\log \frac{d^{\pi_0}(x)}{d^\mu(x)} \le \epsilon\;.
$$
\end{assumption}
This assumption can be enforced by initializing $\pi_0$ to be a small perturbation of $\mu$. 
The constrained optimization problem defined in Eq.~\ref{eq:proximal_regularization} ensures that every $\pi$ encountered during iterations of SHPI stays within the proximal set. 
However since the Lagrangian relaxation in Eq.~\ref{eq:proximal_regularization_oracle} is only an approximation of the constrained problem in  Eq.~\ref{eq:proximal_regularization}, we may potentially violate the proximal constraint. We find that this scheme still works well in our experiments (see below) and we conjecture that this is because in real-world problems  Assumption~\ref{ass:coverage} holds for many more policies $\pi \in \Pi$. 
We have already shown that a contextual bandit satisfies this assumption $\forall \pi \in \Pi$ trivially with $\epsilon=0$. Assumption~\ref{ass:coverage} with $\pi^*$ is reasonable for recommendation scenarios, where $\mu$ may have explored many sequences of recommendations and we seek to find more optimal subsequences. 
In practice, if we find DualDICE weights that are too large during any step of \method, we conclude that distribution matching regularization is likely unreliable. 

\begin{figure}
    \centering
    \includegraphics[width=0.5\linewidth]{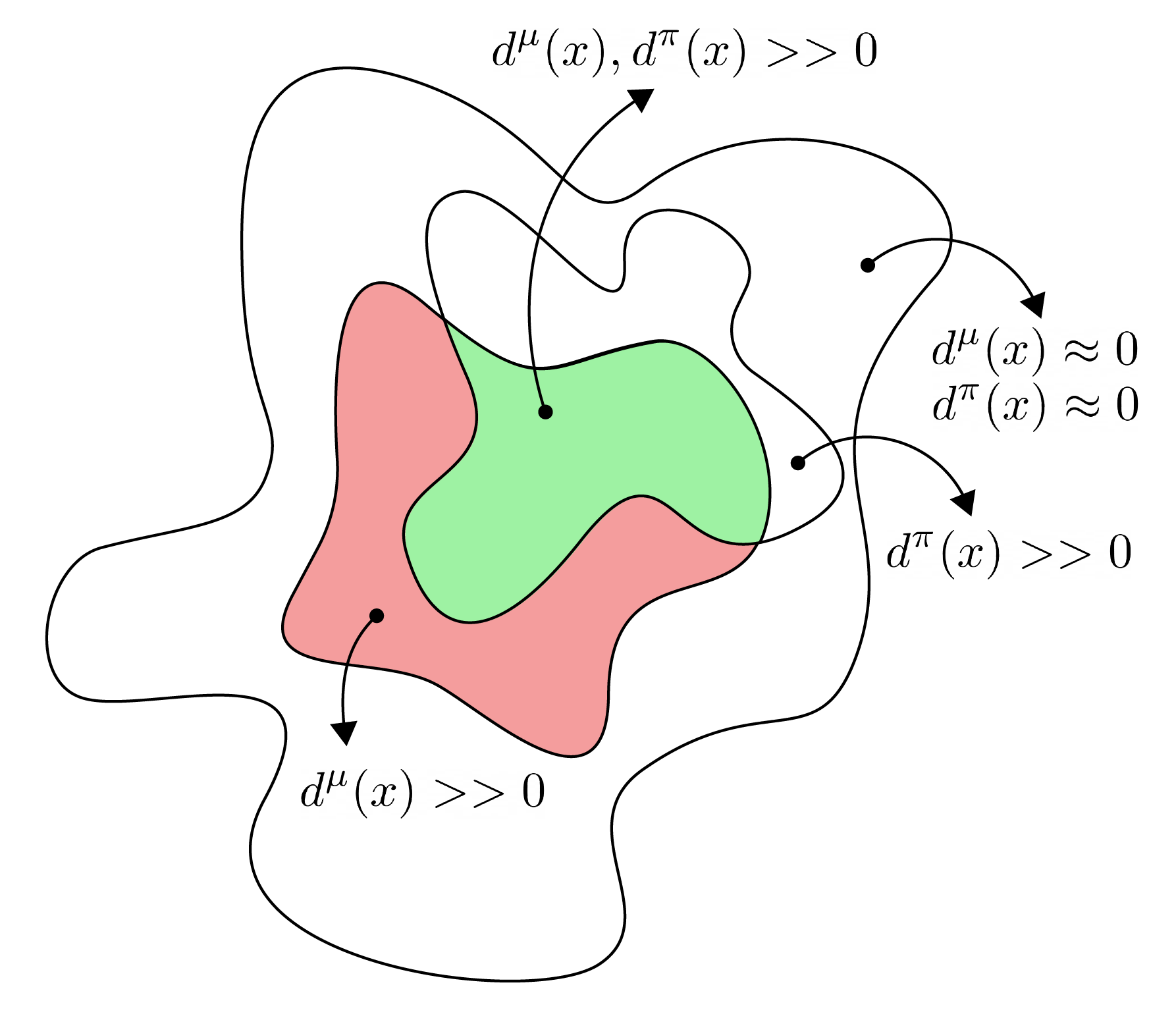}
    \caption{Schematic illustration of proximal regularization: a penalty should be imposed on states which are not properly supported by $d^\mu$. If the initial policy $\pi^{(0)}$ is in the green region (Assumption~\ref{ass:coverage}), then the proximal constraint of Eq.~\ref{eq:proximal_regularization} ensures all policies during approximate policy iteration remain in the green region.}
    \label{fig:proximal_reg}
\end{figure}

\begin{figure}[ht!]
    \centering
    \includegraphics[width=0.8\linewidth]{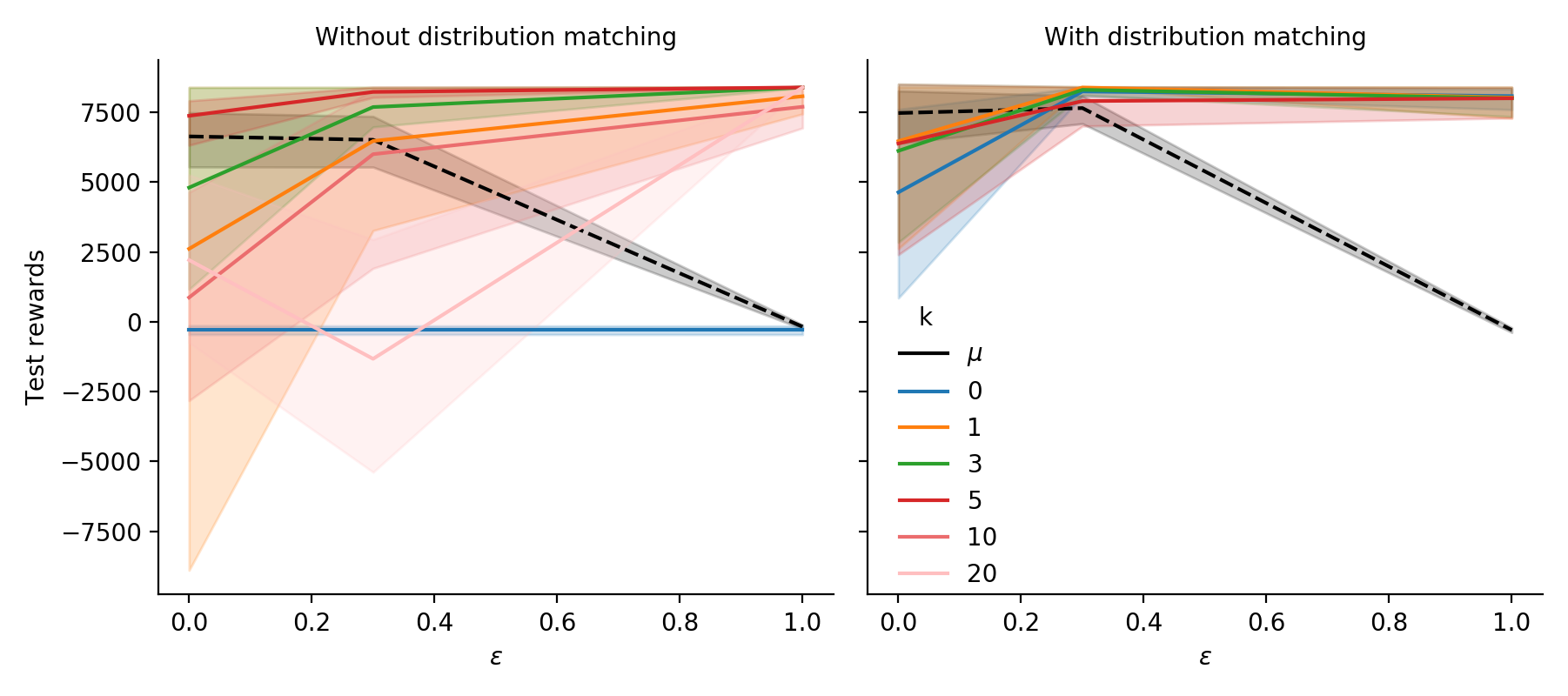}
    \caption{Ablation on parameter $k$ in the synthetic recommendation task with and without the proximal constraint.}
    \label{fig:k_ablation_toy}
\end{figure}

\begin{figure}[ht!]
    \centering
    \includegraphics[width=0.5\linewidth]{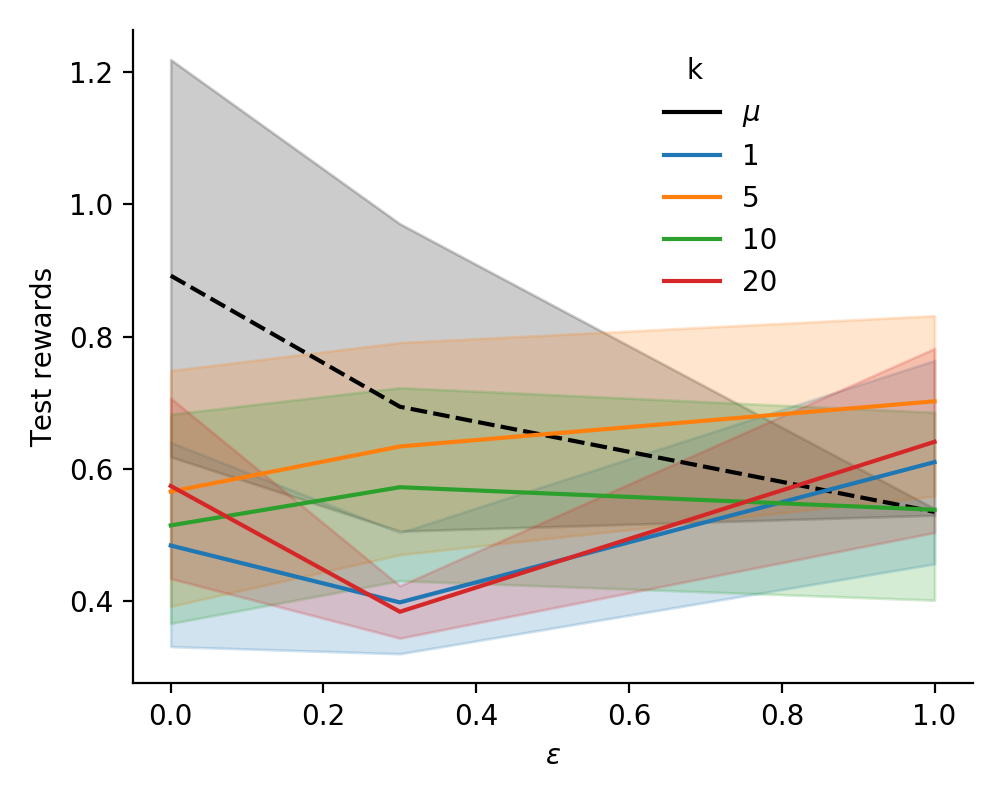}
    \caption{Online ablation on false horizon parameter $k$ for the HIV simulator task.}
    \label{fig:k_ablation_hiv}
\end{figure}

We provide in Figure~\ref{fig:k_ablation_hiv} an additional ablation on the parameter $k$. 

%% file: main.bbl
\begin{thebibliography}{42}
\providecommand{\natexlab}[1]{#1}

\bibitem[{Bertsekas and Tsitsiklis(1995)}]{bertsekas1995neuro}
Bertsekas, D.~P.; and Tsitsiklis, J.~N. 1995.
\newblock Neuro-dynamic programming: an overview.
\newblock In \emph{Proceedings of 1995 34th IEEE conference on decision and
  control}, volume~1, 560--564. IEEE.

\bibitem[{Blackwell(1962)}]{blackwell1962discrete}
Blackwell, D. 1962.
\newblock Discrete dynamic programming.
\newblock \emph{The Annals of Mathematical Statistics}, 719--726.

\bibitem[{Cheng, Kolobov, and Swaminathan(2021)}]{HuRL}
Cheng, C.; Kolobov, A.; and Swaminathan, A. 2021.
\newblock Heuristic-Guided Reinforcement Learning.
\newblock \emph{CoRR}, abs/2106.02757.

\bibitem[{Cheng, Kolobov, and Agarwal(2020)}]{cheng2020policy}
Cheng, C.-A.; Kolobov, A.; and Agarwal, A. 2020.
\newblock Policy Improvement via Imitation of Multiple Oracles.
\newblock \emph{Advances in Neural Information Processing Systems}, 33.

\bibitem[{Dann and Brunskill(2015)}]{dann2015sample}
Dann, C.; and Brunskill, E. 2015.
\newblock Sample complexity of episodic fixed-horizon reinforcement learning.
\newblock \emph{arXiv preprint arXiv:1510.08906}.

\bibitem[{Ernst et~al.(2006)Ernst, Stan, Goncalves, and
  Wehenkel}]{ernst2006clinical}
Ernst, D.; Stan, G.-B.; Goncalves, J.; and Wehenkel, L. 2006.
\newblock Clinical data based optimal STI strategies for HIV: a reinforcement
  learning approach.
\newblock In \emph{Proceedings of the 45th IEEE Conference on Decision and
  Control}, 667--672. IEEE.

\bibitem[{Fujimoto, Meger, and Precup(2019)}]{fujimoto2019off}
Fujimoto, S.; Meger, D.; and Precup, D. 2019.
\newblock Off-policy deep reinforcement learning without exploration.
\newblock In \emph{International Conference on Machine Learning}, 2052--2062.
  PMLR.

\bibitem[{Garg et~al.(2020)Garg, Gupta, Malhotra, Vig, and
  Shroff}]{garg2020batch}
Garg, D.; Gupta, P.; Malhotra, P.; Vig, L.; and Shroff, G. 2020.
\newblock Batch-Constrained Distributional Reinforcement Learning for
  Session-based Recommendation.
\newblock \emph{arXiv preprint arXiv:2012.08984}.

\bibitem[{Hordijk and Yushkevich(2002)}]{hordijk2002blackwell}
Hordijk, A.; and Yushkevich, A.~A. 2002.
\newblock Blackwell optimality.
\newblock In \emph{Handbook of Markov decision processes}, 231--267. Springer.

\bibitem[{Ionides(2008)}]{ionides2008truncated}
Ionides, E.~L. 2008.
\newblock Truncated importance sampling.
\newblock \emph{Journal of Computational and Graphical Statistics}, 17(2):
  295--311.

\bibitem[{Jannach and Ludewig(2017)}]{jannach2017recurrent}
Jannach, D.; and Ludewig, M. 2017.
\newblock When recurrent neural networks meet the neighborhood for
  session-based recommendation.
\newblock In \emph{Proceedings of the Eleventh ACM Conference on Recommender
  Systems}, 306--310.

\bibitem[{Jiang et~al.(2015)Jiang, Kulesza, Singh, and
  Lewis}]{jiang2015dependence}
Jiang, N.; Kulesza, A.; Singh, S.; and Lewis, R. 2015.
\newblock The dependence of effective planning horizon on model accuracy.
\newblock In \emph{Proceedings of the 2015 International Conference on
  Autonomous Agents and Multiagent Systems}, 1181--1189. Citeseer.

\bibitem[{Kakade and Langford(2002)}]{kakade2002approximately}
Kakade, S.; and Langford, J. 2002.
\newblock Approximately optimal approximate reinforcement learning.
\newblock In \emph{ICML}, volume~2, 267--274.

\bibitem[{Kakade(2003)}]{kakade2003sample}
Kakade, S.~M. 2003.
\newblock \emph{On the sample complexity of reinforcement learning}.
\newblock Ph.D. thesis, UCL (University College London).

\bibitem[{Kalweit, Huegle, and Boedecker(2019)}]{kalweit2019composite}
Kalweit, G.; Huegle, M.; and Boedecker, J. 2019.
\newblock Composite Q-learning: Multi-scale Q-function Decomposition and
  Separable Optimization.
\newblock \emph{arXiv preprint arXiv:1909.13518}.

\bibitem[{Kumar et~al.(2020)Kumar, Zhou, Tucker, and
  Levine}]{kumar2020conservative}
Kumar, A.; Zhou, A.; Tucker, G.; and Levine, S. 2020.
\newblock Conservative q-learning for offline reinforcement learning.
\newblock \emph{arXiv preprint arXiv:2006.04779}.

\bibitem[{Laroche, Trichelair, and Des~Combes(2019)}]{laroche2019safe}
Laroche, R.; Trichelair, P.; and Des~Combes, R.~T. 2019.
\newblock Safe policy improvement with baseline bootstrapping.
\newblock In \emph{International Conference on Machine Learning}, 3652--3661.
  PMLR.

\bibitem[{Li et~al.(2017)Li, Ren, Chen, Ren, Lian, and Ma}]{li2017neural}
Li, J.; Ren, P.; Chen, Z.; Ren, Z.; Lian, T.; and Ma, J. 2017.
\newblock Neural attentive session-based recommendation.
\newblock In \emph{Proceedings of the 2017 ACM on Conference on Information and
  Knowledge Management}, 1419--1428.

\bibitem[{Liao et~al.(2020)Liao, Greenewald, Klasnja, and
  Murphy}]{liao2020personalized}
Liao, P.; Greenewald, K.; Klasnja, P.; and Murphy, S. 2020.
\newblock Personalized heartsteps: A reinforcement learning algorithm for
  optimizing physical activity.
\newblock \emph{Proceedings of the ACM on Interactive, Mobile, Wearable and
  Ubiquitous Technologies}, 4(1): 1--22.

\bibitem[{Liu et~al.(2018{\natexlab{a}})Liu, Zeng, Mokhosi, and
  Zhang}]{liu2018stamp}
Liu, Q.; Zeng, Y.; Mokhosi, R.; and Zhang, H. 2018{\natexlab{a}}.
\newblock STAMP: short-term attention/memory priority model for session-based
  recommendation.
\newblock In \emph{Proceedings of the 24th ACM SIGKDD International Conference
  on Knowledge Discovery \& Data Mining}, 1831--1839.

\bibitem[{Liu et~al.(2018{\natexlab{b}})Liu, Gottesman, Raghu, Komorowski,
  Faisal, Doshi-Velez, and Brunskill}]{liu2018representation}
Liu, Y.; Gottesman, O.; Raghu, A.; Komorowski, M.; Faisal, A.~A.; Doshi-Velez,
  F.; and Brunskill, E. 2018{\natexlab{b}}.
\newblock Representation balancing mdps for off-policy policy evaluation.
\newblock \emph{Advances in Neural Information Processing Systems}, 31:
  2644--2653.

\bibitem[{Liu et~al.(2020)Liu, Swaminathan, Agarwal, and
  Brunskill}]{liu2020provably}
Liu, Y.; Swaminathan, A.; Agarwal, A.; and Brunskill, E. 2020.
\newblock Provably good batch reinforcement learning without great exploration.
\newblock \emph{arXiv preprint arXiv:2007.08202}.

\bibitem[{Ludewig and Jannach(2018)}]{ludewig2018evaluation}
Ludewig, M.; and Jannach, D. 2018.
\newblock Evaluation of session-based recommendation algorithms.
\newblock \emph{User Modeling and User-Adapted Interaction}, 28(4-5): 331--390.

\bibitem[{Ma et~al.(2020{\natexlab{a}})Ma, Zhao, Yi, Yang, Chen, Tang, Hong,
  and Chi}]{ma2020off}
Ma, J.; Zhao, Z.; Yi, X.; Yang, J.; Chen, M.; Tang, J.; Hong, L.; and Chi,
  E.~H. 2020{\natexlab{a}}.
\newblock Off-policy learning in two-stage recommender systems.
\newblock In \emph{Proceedings of The Web Conference 2020}, 463--473.

\bibitem[{Ma et~al.(2020{\natexlab{b}})Ma, Narayanaswamy, Lin, and
  Ding}]{ma2020temporal}
Ma, Y.; Narayanaswamy, B.; Lin, H.; and Ding, H. 2020{\natexlab{b}}.
\newblock Temporal-Contextual Recommendation in Real-Time.
\newblock In \emph{Proceedings of the 26th ACM SIGKDD International Conference
  on Knowledge Discovery \& Data Mining}, 2291--2299.

\bibitem[{Nachum et~al.(2019)Nachum, Chow, Dai, and Li}]{nachum2019dualdice}
Nachum, O.; Chow, Y.; Dai, B.; and Li, L. 2019.
\newblock Dualdice: Behavior-agnostic estimation of discounted stationary
  distribution corrections.
\newblock \emph{arXiv preprint arXiv:1906.04733}.

\bibitem[{Pan et~al.(2019)Pan, Cai, Tang, Zhuang, and He}]{pan2019policy}
Pan, F.; Cai, Q.; Tang, P.; Zhuang, F.; and He, Q. 2019.
\newblock Policy gradients for contextual recommendations.
\newblock In \emph{The World Wide Web Conference}, 1421--1431.

\bibitem[{Precup(2000)}]{precup2000eligibility}
Precup, D. 2000.
\newblock Eligibility traces for off-policy policy evaluation.
\newblock \emph{Computer Science Department Faculty Publication Series}, 80.

\bibitem[{Rohde et~al.(2018)Rohde, Bonner, Dunlop, Vasile, and
  Karatzoglou}]{rohde2018recogym}
Rohde, D.; Bonner, S.; Dunlop, T.; Vasile, F.; and Karatzoglou, A. 2018.
\newblock Recogym: A reinforcement learning environment for the problem of
  product recommendation in online advertising.
\newblock \emph{arXiv preprint arXiv:1808.00720}.

\bibitem[{Romoff et~al.(2019)Romoff, Henderson, Touati, Brunskill, Pineau, and
  Ollivier}]{romoff2019separating}
Romoff, J.; Henderson, P.; Touati, A.; Brunskill, E.; Pineau, J.; and Ollivier,
  Y. 2019.
\newblock Separating value functions across time-scales.
\newblock In \emph{International Conference on Machine Learning}, 5468--5477.
  PMLR.

\bibitem[{Schulman et~al.(2015{\natexlab{a}})Schulman, Levine, Abbeel, Jordan,
  and Moritz}]{schulman2015trust}
Schulman, J.; Levine, S.; Abbeel, P.; Jordan, M.; and Moritz, P.
  2015{\natexlab{a}}.
\newblock Trust region policy optimization.
\newblock In \emph{International conference on machine learning}, 1889--1897.

\bibitem[{Schulman et~al.(2015{\natexlab{b}})Schulman, Moritz, Levine, Jordan,
  and Abbeel}]{schulman2015high}
Schulman, J.; Moritz, P.; Levine, S.; Jordan, M.; and Abbeel, P.
  2015{\natexlab{b}}.
\newblock High-dimensional continuous control using generalized advantage
  estimation.
\newblock \emph{arXiv preprint arXiv:1506.02438}.

\bibitem[{Schulman et~al.(2017)Schulman, Wolski, Dhariwal, Radford, and
  Klimov}]{schulman2017proximal}
Schulman, J.; Wolski, F.; Dhariwal, P.; Radford, A.; and Klimov, O. 2017.
\newblock Proximal policy optimization algorithms.
\newblock \emph{arXiv preprint arXiv:1707.06347}.

\bibitem[{Shani et~al.(2005)Shani, Heckerman, Brafman, and
  Boutilier}]{shani2005mdp}
Shani, G.; Heckerman, D.; Brafman, R.~I.; and Boutilier, C. 2005.
\newblock An MDP-based recommender system.
\newblock \emph{Journal of Machine Learning Research}, 6(9).

\bibitem[{Styblinski and Tang(1990)}]{styblinski1990experiments}
Styblinski, M.; and Tang, T.-S. 1990.
\newblock Experiments in nonconvex optimization: stochastic approximation with
  function smoothing and simulated annealing.
\newblock \emph{Neural Networks}, 3(4): 467--483.

\bibitem[{Su, Srinath, and Krishnamurthy(2020)}]{su2020adaptive}
Su, Y.; Srinath, P.; and Krishnamurthy, A. 2020.
\newblock Adaptive Estimator Selection for Off-Policy Evaluation.
\newblock \emph{arXiv preprint arXiv:2002.07729}.

\bibitem[{Sun, Bagnell, and Boots(2018)}]{sun2018truncated}
Sun, W.; Bagnell, J.~A.; and Boots, B. 2018.
\newblock Truncated horizon policy search: Combining reinforcement learning \&
  imitation learning.
\newblock \emph{arXiv preprint arXiv:1805.11240}.

\bibitem[{Vernade, Gyorgy, and Mann(2020)}]{vernadenon}
Vernade, C.; Gyorgy, A.; and Mann, T. 2020.
\newblock Non-Stationary Delayed Bandits with Intermediate Observations.
\newblock In \emph{International Conference on Machine Learning}, 9722--9732.
  PMLR.

\bibitem[{Wu et~al.(2017)Wu, Wang, Hong, and Shi}]{wu2017returning}
Wu, Q.; Wang, H.; Hong, L.; and Shi, Y. 2017.
\newblock Returning is believing: Optimizing long-term user engagement in
  recommender systems.
\newblock In \emph{Proceedings of the 2017 ACM on Conference on Information and
  Knowledge Management}, 1927--1936.

\bibitem[{Yin, Bai, and Wang(2021)}]{yin2021near}
Yin, M.; Bai, Y.; and Wang, Y.-X. 2021.
\newblock Near-Optimal Offline Reinforcement Learning via Double Variance
  Reduction.
\newblock \emph{arXiv preprint arXiv:2102.01748}.

\bibitem[{Zhang et~al.(2020)Zhang, Zhao, Guan, Chen, Bian, Song, Cui, and
  Li}]{zhang2020preference}
Zhang, Y.; Zhao, P.; Guan, Y.; Chen, L.; Bian, K.; Song, L.; Cui, B.; and Li,
  X. 2020.
\newblock Preference-aware mask for session-based recommendation with
  bidirectional transformer.
\newblock In \emph{ICASSP 2020-2020 IEEE International Conference on Acoustics,
  Speech and Signal Processing (ICASSP)}, 3412--3416. IEEE.

\bibitem[{Zheng et~al.(2018)Zheng, Zhang, Zheng, Xiang, Yuan, Xie, and
  Li}]{zheng2018drn}
Zheng, G.; Zhang, F.; Zheng, Z.; Xiang, Y.; Yuan, N.~J.; Xie, X.; and Li, Z.
  2018.
\newblock DRN: A deep reinforcement learning framework for news recommendation.
\newblock In \emph{Proceedings of the 2018 World Wide Web Conference},
  167--176.

\end{thebibliography}
